\documentclass{article}

\usepackage{times}
\usepackage{graphicx} 

\usepackage{natbib}

\usepackage{algorithm}
\usepackage{algorithmic}
\usepackage{amsmath,amssymb,amsthm}

\usepackage{hyperref}

\usepackage{caption}
\usepackage{subcaption}
\usepackage{tikz}
\usepackage{xcolor}
\usetikzlibrary{arrows}
\usepackage{threeparttable}
\usepackage{multirow}
\usepackage{cancel}
\usepackage{tabularx}

\usepackage{paralist}
\usepackage{todonotes}

\usepackage[accepted]{icml2017}

\newcommand{\nn}{\nonumber}
\newcommand{\defeq}{\triangleq}
\newcommand{\mc}{\mathcal}
\newcommand{\mb}{\mathbb}

\newcommand{\diag}{\mathrm{diag}}

\newcommand{\E}{\mathbb{E}}

\def\R{\mathbb{R}}

\def\E{\mathbb{E}}

\newcommand{\mat}[1]{{#1}}

\newcommand{\norm}[1]{\left\|#1\right\|}
\newcommand{\abs}[1]{\left|#1\right|}
\newcommand{\expect}[1]{\mathbb{E}\left[#1\right]}

\newtheorem{thm}{Theorem}
\newtheorem{lem}{Lemma}
\newtheorem{asmp}{Assumption}

\allowdisplaybreaks

 \usepackage[noindentafter]{titlesec}
 \titlespacing\section{0pt}{4pt plus 0pt minus 2pt}{2pt plus 0pt minus 1pt}
 \titlespacing\subsection{0pt}{4pt plus 0pt minus 1pt}{1pt plus 1pt minus 1pt}
 \titlespacing\subsubsection{0pt}{4pt plus 0pt minus 1pt}{1pt plus 1pt minus 1pt}

\icmltitlerunning{Stochastic Variance Reduction Methods for Policy Evaluation}

\begin{document} 

\twocolumn[
\icmltitle{Stochastic Variance Reduction Methods for Policy Evaluation}




\begin{icmlauthorlist}
\icmlauthor{Simon S. Du}{cmu}
\icmlauthor{Jianshu Chen}{msr}
\icmlauthor{Lihong Li}{msr}
\icmlauthor{Lin Xiao}{msr}
\icmlauthor{Dengyong Zhou}{msr}
\end{icmlauthorlist}

\icmlaffiliation{cmu}{Machine Learning Department,  
                      Carnegie Mellon University,
                      Pittsburgh, Pennsylvania 15213, USA.}
\icmlaffiliation{msr}{Microsoft Research, Redmond, Washington 98052, USA.}

\icmlcorrespondingauthor{Simon S. Du}{{ssdu@cs.cmu.edu}}
\icmlcorrespondingauthor{Jianshu Chen}{{jianshuc@microsoft.com}}
\icmlcorrespondingauthor{Lihong Li}{{lihongli@microsoft.com}}
\icmlcorrespondingauthor{Lin Xiao}{{lin.xiao@microsoft.com}}
\icmlcorrespondingauthor{Dengyong Zhou}{{denzho@microsoft.com}}

\icmlkeywords{boring formatting information, machine learning, ICML}

\vskip 0.3in
]



\printAffiliationsAndNotice{}  

\begin{abstract} 
Policy evaluation is concerned with estimating the value function that predicts long-term values of states under a given policy.  It is a crucial step in many reinforcement-learning algorithms.  In this paper, we focus on policy evaluation with linear function approximation over a \emph{fixed} dataset.
We first transform the empirical policy evaluation problem into a
(quadratic) convex-concave saddle-point problem, and then present a primal-dual batch gradient method, as well as two stochastic variance reduction methods for solving the problem.
These algorithms scale linearly in both sample size and feature dimension.
Moreover, they achieve \emph{linear} convergence even when the saddle-point problem has only strong concavity in the dual variables 
but \emph{no} strong convexity in the primal variables. 
Numerical experiments on benchmark problems
demonstrate the effectiveness of our methods.
\end{abstract} 

\section{Introduction}
\label{sec:intro}

Reinforcement learning (RL) is a powerful learning paradigm for 
sequential decision making 
\citep[see, e.g.,][]{bertsekas1995neuro,sutton1998reinforcement}.
An RL agent interacts with the environment by repeatedly
observing the current state, taking an action according to a certain policy,
receiving a reward signal and transitioning to a next state. 
A policy specifies which action to take given the current state. 
\emph{Policy evaluation} estimates a value function that predicts expected
cumulative reward the agent would receive by following a fixed policy
starting at a certain state.
In addition to quantifying long-term values of states, which can be of interest on its own,
value functions also provide important information for the agent to optimize its policy.
For example, \emph{policy-iteration} algorithms iterate between policy-evaluation steps 
and policy-\emph{improvement} steps, until a (near-)optimal policy is found~\cite{bertsekas1995neuro,lagoudakis2003least}.
Therefore, estimating the value function efficiently and accurately is essential in RL.

There has been substantial work on policy evaluation, with \emph{temporal-difference} (TD) methods being perhaps the most popular.
These methods use the Bellman equation to bootstrap the estimation process. 
Different cost functions are formulated to exploit this idea, 
leading to different policy evaluation algorithms; 
%
see \citet{dann2014policy} for a comprehensive survey.
In this paper, we study policy evaluation by minimizing 
the mean squared projected Bellman error (MSPBE) 
with linear approximation of the value function.
We focus on the batch setting where a fixed, finite dataset is given.
This fixed-data setting is not only important in itself~\cite{Lange11Batch}, but also an important component in other
RL methods such as \emph{experience replay}~\cite{Lin92Self}.

The finite-data regime makes it possible to solve policy evaluation more efficiently with recently developed fast optimization methods based on 
\emph{stochastic variance reduction}, such as 
SVRG~\cite{johnson2013accelerating} and SAGA~\cite{defazio2014saga}.
For minimizing strongly convex functions with a finite-sum structure, such methods enjoy the same low computational cost per iteration as the classical stochastic gradient method, but also achieve fast, linear convergence rates (i.e., exponential decay of the optimality gap in the objective).  However, they cannot be applied directly to minimize the MSPBE, whose objective does not have the finite-sum structure.
In this paper,
we overcome this obstacle by transforming the empirical MSPBE problem 
to an \emph{equivalent} convex-concave saddle-point problem that possesses the desired finite-sum structure.

In the saddle-point problem, we consider the model parameters as the primal
variables, which are coupled with the dual variables through a bilinear term. 
Moreover, without an $\ell_2$-regularization on the model parameters, the objective 
is only strongly concave in the dual variables,
but \emph{not} 
in the primal variables. 
We propose a primal-dual batch gradient method,
as well as two stochastic variance-reduction methods based on
SVRG and SAGA, respectively.
Surprisingly, we show that when the coupling matrix is full rank, these algorithms 
achieve linear convergence in both the primal and dual spaces, despite the lack of strong convexity of the objective in the primal variables.
Our results also extend to \emph{off-policy} learning and TD with
\emph{eligibility traces}~\cite{sutton1998reinforcement,Precup01Off}.

We note that \citet{balamurugan2016stochastic} have extended 
both SVRG and SAGA to solve convex-concave saddle-point problems with
linear-convergence guarantees. The main difference between our results and theirs
are
\begin{compactitem}
  \item Linear convergence in \citet{balamurugan2016stochastic} relies on the assumption that the objective is strongly convex in the primal variables and strongly concave in the dual.
    Our results show, somewhat surprisingly, that only one of them is necessary if the primal-dual
    coupling is bilinear and the coupling matrix is full rank. 
    In fact, we are not aware of similar previous results even for the 
    primal-dual batch gradient method, which we show in this paper.
  \item 
    Even if a strongly convex regularization on the primal variables is introduced to the MSPBE objective,
    the algorithms in \citet{balamurugan2016stochastic} 
    cannot be applied efficiently.
    Their algorithms require that the proximal mappings of 
    the strongly convex and concave regularization functions 
    be computed efficiently.
    In our saddle-point formulation, the strong concavity of the dual variables
    comes from a quadratic function defined by the feature covariance matrix,
    which cannot be inverted efficiently and makes the proximal mapping costly
    to compute.
    Instead, our algorithms only use its (stochastic) gradients and hence are much more efficient.
\end{compactitem}
We compare various gradient based algorithms on a Random MDP and Mountain Car 
data sets. 
The experiments demonstrate the effectiveness of our proposed methods.

\section{Preliminaries}
\label{sec:pre}

%
We consider a \emph{Markov Decision Process} 
(MDP)~\citep{puterman2005markov} described by
$\left(\mathcal{S},\mathcal{A},\mathcal{P}_{ss'}^a,\mathcal{R},\gamma\right)$, 
where $\mathcal{S}$ is the set of states, 
$\mathcal{A}$ the set of actions,
$\mathcal{P}_{ss'}^a$ the transition probability from state 
$s$ to state $s'$ after taking action 
$a$, $\mathcal{R}\left(s,a\right)$ the reward received after taking action $a$ in state $s$, and $\gamma \in [0,1)$ a discount factor.
The goal of an agent is to find an action-selection policy $\pi$, 
so that the long-term reward under this policy is maximized.  For ease of exposition, we assume $\mathcal{S}$ is finite, but none of our results relies on this assumption.

A key step in many algorithms in RL is to estimate the value function of 
a given policy $\pi$, defined as
$V^{\pi}(s)	\defeq	
					\E [
						\sum_{t=0}^{\infty}
                        \gamma^t \mathcal{R}(s_t,a_t) 
						| s_0 = s, \pi
					]$.
%
Let $V^{\pi}$ denote a vector constructed by stacking the values of $V^{\pi}(1),\ldots, V^{\pi}(|\mc{S}|)$ on top of each other. Then $V^\pi$ is the unique fixed point of the \emph{Bellman operator} $T^\pi$:
	\begin{align}
		V^\pi = T^\pi V^\pi \triangleq R^\pi + \gamma P^\pi V^\pi\,,	\label{Equ:MSPBE:BellmanEq}
	\end{align} 
where $R^\pi$ is the expected reward vector under policy~$\pi$,
defined elementwise as $R^\pi(s)=\E_{\pi(a|s)} \mathcal{R}(s,a)$;
and $P^\pi$ is the transition matrix induced by the policy applying $\pi$,
defined entrywise as $P^\pi(s,s')=\E_{\pi(a|s)} \mathcal{P}^a_{ss'}$.

\subsection{Mean squared projected Bellman error (MSPBE)}

One approach to scale up when the state space size $\abs{\mathcal{S}}$ is large or infinite is to use a linear approximation for $V^\pi$.
Formally, we use a feature map $\phi: \mathcal{S} \rightarrow \mathbb{R}^d$ and
approximate the value function by 
$\widehat{V}^\pi\left(s\right) = \phi(s)^T \theta$,
where $\theta \in \mathbb{R}^d$ is the model parameter to be estimated. 
Here, we want to find $\theta$ that minimizes 
the mean squared projected Bellman error, or MSPBE:
\begin{align}
		\text{MSPBE}\left(\theta\right) 
			&\defeq 
				\frac{1}{2}
				\|\widehat{V}^\pi - \Pi T^\pi \widehat{V}^\pi\|_{\Xi}^2 ,				\label{eqn:mspbe}
\end{align}
where $\Xi$ is a diagonal matrix with diagonal elements being the stationary
distribution over $\mathcal{S}$ induced by the policy~$\pi$, and $\Pi$ is the
weighted projection matrix onto the linear space spanned by
$\phi(1),\ldots, \phi(|\mc{S}|)$, that is, 
\begin{align}
		\Pi
			&=
				\Phi (\Phi^T \Xi \Phi)^{-1} \Phi^T \Xi 
		\label{Equ:MSPBE:Pi}
\end{align} 
where $\Phi \defeq [ \phi^T(1),\ldots,\phi^T(|\mc{S}|)]$ is the matrix obtained
by stacking the feature vectors row by row.
Substituting \eqref{Equ:MSPBE:Pi} and~\eqref{Equ:MSPBE:BellmanEq} 
into~\eqref{eqn:mspbe},
we obtain \citep[see, e.g.,][]{dann2014policy}
\begin{align*} 
		\text{MSPBE}(\theta)
			&=
				\frac{1}{2}
				\|
					\Phi^T \Xi (\widehat{V}^{\pi} - T^{\pi} \widehat{V}^{\pi} ) 
				\|^2_{(\Phi^T \Xi \Phi)^{-1}} .
\end{align*}
We can further rewrite the above expression for MSPBE as a standard weighted 
least-squares problem:
\begin{align*} 
		\text{MSPBE}(\theta)
			&=
				\frac{1}{2}
				\| 
					A \theta - b
				\|_{C^{-1}}^2,
\end{align*}
with properly defined $A$, $b$ and $C$, described as follows.
Suppose the MDP under policy~$\pi$ settles at its stationary distribution 
and generates an infinite transition sequence
$\left\{\left(s_t, a_t,r_t,s_{t+1}\right)\right\}_{t=1}^\infty$,
where $s_t$ is the current state, $a_t$ is the action, 
$r_t$ is the reward, and $s_{t+1}$ is the next state.
Then with the definitions
$\phi_t \defeq \phi(s_t)$ and $\phi_t' \defeq \phi(s_{t+1})$,
we have
\begin{align}\label{eqn:AbC}
   A = \E[ \phi_t (\phi_t - \gamma \phi_t')^T ], 
  ~b = \E[ \phi_t r_t ], 
  ~C = \E [\phi_t \phi_t^T ],
\end{align}
where $\E[\cdot]$ are with respect to the stationary distribution.  Many TD  solutions converge to a minimizer of MSPBE in the limit~\cite{Tsitsiklis97Analysis,dann2014policy}.

\subsection{Empirical MSPBE}
\label{sec:em_mspbe}

In practice, quantities in \eqref{eqn:AbC} are often unknown, and we only have access to a finite dataset with $n$ transitions $\mathcal{D} = \left\{\left(s_t, a_t,r_t,s_{t+1}\right)\right\}_{t=1}^n$.
%
%
By replacing the unknown statistics with their finite-sample estimates, we obtain the Empirical MSPBE, or EM-MSPBE.  Specifically, let
\begin{align}
	\widehat{\mat{A}} 
		& \triangleq  \frac{1}{n}\sum_{t=1}^{n} A_t, \quad
	\widehat{b} 
		\triangleq \frac{1}{n}\sum_{t=1}^{n} b_t, \quad
	\widehat{\mat{C}} 
		\triangleq \frac{1}{n}\sum_{t=1}^{n}\mat{C}_t,
	\label{Equ:MSPBE:Ahat_def}
\end{align}
where for $t=1,\ldots,n$, 
\begin{equation}\label{eqn:def-At-bt-Ct}
  A_t \triangleq \phi_t(\phi_t-\gamma\phi'_t)^T, \quad
  b_t \triangleq r_t\phi_t, \quad 
  C_t \triangleq \phi_t\phi_t^T.
\end{equation}
EM-MSPBE with an \emph{optional} $\ell_2$-regularization is given by:
\begin{align}
	\text{EM-MSPBE}\left(\theta\right) 
		&= 
			\frac{1}{2}
			\| \widehat{\mat{A}}\theta - \widehat{b}\|_{\widehat{\mat{C}}^{-1}}^2
			+
			\frac{\rho}{2}
			\| \theta \|^2 ,
	\label{eqn:em-mspbe}
\end{align}
where $\rho \ge 0$ is a regularization factor.

Observe that \eqref{eqn:em-mspbe} is a (regularized) weighted least squares problem. 
Assuming $\widehat{C}$ is invertible, its optimal solution is
\begin{align}\label{eqn:theta-star}
  \theta^\star = (\widehat{A}^\top
\widehat{C}^{-1} \widehat{A} + \rho I)^{-1} \widehat{A}^\top \widehat{C}^{-1}
\widehat{b}.
\end{align}
Computing $\theta^\star$ directly requires $O(nd^2)$ operations to form the 
matrices $\widehat{A}$, $\widehat{b}$ and $\widehat{C}$, and then 
$O(d^3)$ operations to complete the calculation.
This method, known as least-squares temporal difference or LSTD~\citep{BradtkeBarto1996LSTD,Boyan2002LSTD},
can be very expensive when $n$ and $d$ are large.
One can also skip forming the matrices explicitly and
compute $\theta^\star$ using~$n$ recusive rank-one updates 
\citep{NedicBertsekas2003}.
Since each rank-one update costs $O(d^2)$, the total cost is $O(nd^2)$.

In the sequel, we develop efficient algorithms to minimize
EM-MSPBE by using stochastic variance reduction methods, which samples one
$(\phi_t, \phi_t')$ per update without pre-computing $\widehat{A}$,
$\widehat{b}$ and $\widehat{C}$.  These algorithms not only maintain a low $O(d)$ per-iteration computation cost, but also attain fast linear convergence rates with a $\log(1/\epsilon)$ dependence 
on the desired accuracy $\epsilon$.

\section{Saddle-Point Formulation of EM-MSPBE}
\label{sec:em_saddle}

Our algorithms (in Section~\ref{sec:algos}) are based on the stochastic variance
reduction techniques developed for minimizing a finite sum of convex functions, more specifically, SVRG \citep{johnson2013accelerating} and 
SAGA \citep{defazio2014saga}.
They deal with problems of the form
\begin{align}\label{eqn:min-finite-sum}
  \min_{x\in\R^d} ~\biggl\{ f(x) \triangleq \frac{1}{n}\sum_{i=1}^n f_i(x) \biggr\},
\end{align}
where each $f_i$ is convex.
We immediately notice that the EM-MSPBE in~\eqref{eqn:em-mspbe} \emph{cannot} 
be put into such a form, even though the matrices $\widehat{A}$, $\widehat{b}$
and $\widehat{C}$ have the finite-sum structure given in~\eqref{Equ:MSPBE:Ahat_def}.
Thus, extending variance reduction techniques to EM-MSPBE minimization
is not straightforward.

Nevertheless, we will show that the minimizing the EM-MSPBE is equivalent to
solving a convex-concave saddle-point problem which actually possesses the
desired finite-sum structure.
To proceed, we resort to the machinery of \emph{conjugate functions}
\citep[e.g.][Section~12]{Rockafellar70}. 
For a function $f: \mb{R}^d \rightarrow \mb{R}$, its
conjugate function $f^{\star}: \mb{R}^d \rightarrow \mb{R}$ is defined as
$f^{\star}(y) \defeq \sup_{x} (y^T x - f(x))$. 
Note that the conjugate function of $\frac{1}{2} \| x \|_{\widehat{C}}^2$ is
$\frac{1}{2} \| y \|_{\widehat{C}^{-1}}^2$, i.e.,  
\begin{align*}
		\frac{1}{2} \| y \|_{\widehat{C}^{-1}}^2
			&= \max_{x} \Big(y^T x - \frac{1}{2} \| x \|_{\widehat{C}}^2 \Big).
\end{align*}
With this relation, we can rewrite EM-MSPBE in \eqref{eqn:em-mspbe} as
\begin{equation*}
			\max_{w}
				\Big( w^T(\widehat{b}-\widehat{A}\theta)
					- \frac{1}{2} \| w \|_{\widehat{C}}^2
				\Big)
				+
				\frac{\rho}{2}
				\|\theta\|^2\,,
				\nn
\end{equation*}
so that minimizing EM-MSPBE is equivalent to solving
\begin{align}
	\min_{\theta \in \mb{R}^d} \max_{w \in \mb{R}^d}
    ~ \biggl\{ \mathcal{L}(\theta, w) 
    = \frac{1}{n}\sum_{t=1}^n \mathcal{L}_t (\theta, w) \biggr\},
\label{eqn:em_mspbe_saddle_point}
\end{align}
where the Lagrangian, defined as
\begin{align}
  \mathcal{L}(\theta, w) \triangleq
				\frac{\rho}{2}
				\|\theta\|^2
				-
				w^T \widehat{A}\theta
				-
				\Big(
					\frac{1}{2} \| w \|_{\widehat{C}}^2
					-
					w^T \widehat{b}
				\Big)\,,
\end{align}
may be decomposed using  \eqref{Equ:MSPBE:Ahat_def}, with
\[
  \mathcal{L}_t(\theta, w) \defeq 
  \frac{\rho}{2} \|\theta\|^2 - w^T A_t \theta - \Bigl(\frac{1}{2} \|w\|_{C_t}^2 - w^T b_t\Bigr) .
\]
Therefore, minimizing the EM-MSPBE is equivalent to solving the saddle-point
problem~\eqref{eqn:em_mspbe_saddle_point}, which is convex in the primal 
variable~$\theta$ and concave in the dual variable~$w$.
Moreover, it has a finite-sum structure similar to~\eqref{eqn:min-finite-sum}.

\citet{liu2015finite} and \citet{valcarcel2015distributed} independently 
showed that the GTD2 algorithm \cite{sutton2009fast}
is indeed a \emph{stochastic gradient} method for solving the 
saddle-point problem~\eqref{eqn:em_mspbe_saddle_point}, although
they obtained the saddle-point formulation with different derivations.
More recently, \citet{DaiHePanBootsSong2016} used the conjugate function approach
to obtain saddle-point formulations for a more general class of problems
and derived primal-dual stochastic gradient algorithms for solving them.
However, 
these algorithms have sublinear convergence rates, which leaves much room
to improve when applied to problems with finite datasets. Recently, \citet{lian2016finite}
developed SVRG methods for a general finite-sum composition optimization that achieve
linear convergence rate. Different from our methods, their stochastic gradients are biased and they have
worse dependency on the condition numbers ($\kappa^3$ and $\kappa^4$).

The fast linear convergence of our algorithms presented in
Sections~\ref{Sec:BatchGradient} and~\ref{sec:algos}
requires the following assumption:
\begin{asmp}
$\widehat{\mat{A}}$ has full rank, $\widehat{\mat{C}}$ is strictly positive definite, and the feature vector $\phi_t$ is uniformly bounded.
\label{Asmp:nonsingular}
\end{asmp}
Under mild regularity conditions~\citep[e.g.,][Chapter~5]{wasserman2013all}, 
we have $\widehat{\mat{A}}$ and $\widehat{C}$ converge in probability
to $\mat{A}$ and $\mat{C}$ defined in~\eqref{eqn:AbC}, respectively.
Thus, if the true statistics $\mat{A}$ is non-singular and
$\mat{C}$ is positive definite, 
and we have enough training samples, these assumptions are
usually satisfied. 
They have been widely used in previous works on gradient-based algorithms
\citep[e.g.,][]{sutton2009convergent,sutton2009fast}.

A direct consequence of Assumption~\ref{Asmp:nonsingular} is that 
$\theta^\star$ in~\eqref{eqn:theta-star} is the unique minimizer of
the EM-MSPBE in~\eqref{eqn:em-mspbe}, even without any strongly convex
regularization on~$\theta$ (i.e., even if $\rho=0$).
However, if $\rho=0$, then the Lagrangian $\mathcal{L}(\theta,w)$ is 
only strongly concave in~$w$, but not strongly convex in~$\theta$.
In this case, we will show that non-singularity of the coupling 
matrix~$\widehat{A}$ can ``pass'' an implicit strong convexity 
on~$\theta$, which is exploited by our algorithms to obtain 
linear convergence in both the primal and dual spaces.

\section{A Primal-Dual Batch Gradient Method}
\label{Sec:BatchGradient}

Before diving into the stochastic variance reduction algorithms, we first
present Algorithm~\ref{algo:PDBG_saddle}, which is a primal-dual 
\emph{batch} gradient (PDBG) algorithm for solving the
saddle-point problem~\eqref{eqn:em_mspbe_saddle_point}.
In Step~2, the vector $B(\theta,w)$ is obtained by
stacking the primal and negative dual gradients:
\begin{align}
	B\left(\theta,w\right)  \triangleq
					\begin{bmatrix}
						\nabla_{\theta} L(\theta,w) \\
						-\!\nabla_{w} L(\theta,w)
					\end{bmatrix}
                    =
					\begin{bmatrix}
                      \rho\theta - \widehat{A}^T w \\
                      \widehat{A}\theta - \widehat{b} + \widehat{C} w
					\end{bmatrix} .
\label{eqn:em_saddle_gradient}
\end{align}
Some notation is needed in order to characterize the convergence rate of 
Algorithm~\ref{algo:PDBG_saddle}.
For any symmetric and positive definite matrix~$S$, 
let $\lambda_{\max}(S)$ and $\lambda_{\min}(S)$ denote its 
maximum and minimum eigenvalues respectively, and define its condition number to be $\kappa(S)\defeq\lambda_{\max}(S)/\lambda_{\min}(S)$.
We also define $L_\rho$ and $\mu_\rho$ for any $\rho\geq0$:
\begin{align}
  L_\rho &\defeq \lambda_{\max}(\rho I + \widehat{A}^T \widehat{C}^{-1} \widehat{A}), \label{eqn:L-rho} \\
  \mu_\rho &\defeq \lambda_{\min}(\rho I + \widehat{A}^T \widehat{C}^{-1} \widehat{A}). \label{eqn:mu-rho}
\end{align}
By Assumption~\ref{Asmp:nonsingular}, we have $L_\rho\geq \mu_\rho>0$.
The following theorem is proved in Appendix~\ref{sec:PDBG-analysis}.
\begin{thm}\label{thm:pdbg}
Suppose Assumption~\ref{Asmp:nonsingular} holds and let 
$(\theta_{\star}, w_{\star})$ be the (unique) solution 
of~\eqref{eqn:em_mspbe_saddle_point}.
If the step sizes are chosen as
$\sigma_\theta = \frac{1}{ 9 L_\rho \kappa(\widehat{C})}$ and
$\sigma_w = \frac{8}{9\lambda_{\max}(\widehat{C})}$,
then the number of iterations of Algorithm~\ref{algo:PDBG_saddle} 
to achieve $\|\theta-\theta_{\star}\|^2 + \|w - w_{\star}\|^2 \le \epsilon^2$ 
is upper bounded by
\begin{align}
	O
	\left(		
		\kappa
		\left(
			\rho I + \widehat{A}^T \widehat{C}^{-1} \widehat{A}
		\right) 
		\cdot
        \kappa(\widehat{C})
		\cdot
		\log\Bigl(\frac{1}{\epsilon}\Bigr)
	\right) .
	\label{Equ:Thm_PDBG:ComplexityBound}
\end{align}
\end{thm}
We assigned specific values to the step sizes $\sigma_\theta$ and $\sigma_w$
for clarity.
In general, we can use similar step sizes while keeping their ratio
roughly constant as
$\frac{\sigma_w}{\sigma_\theta}\approx\frac{8 L_\rho}{\lambda_{\min}(\widehat{C})}$; see Appendices~\ref{sec:qp_analysis_G} and~\ref{sec:PDBG-analysis} for more details.  In practice, one can use a parameter search on a small subset of data to find reasonable step sizes.  It is an interesting open problem how to automatically select and adjust step sizes.

\begin{algorithm}[tb]
	\renewcommand{\algorithmicrequire}{\textbf{Inputs:}}
	\renewcommand{\algorithmicensure}{\textbf{Outputs:}}
	\caption{PDBG for Policy Evaluation}
	\label{algo:PDBG_saddle}
	\begin{algorithmic}[1]
		\REQUIRE 
        initial point $\left(\theta,w\right)$, step sizes $\sigma_\theta$ and $\sigma_w$, and number of epochs~$M$. 
		\FOR{$i=1$ {\bfseries to} $M$}
				\STATE 		 
                \vspace{0.8ex}
					$
					\begin{bmatrix}
					\theta\\
					w
					\end{bmatrix}
					\leftarrow
					\begin{bmatrix}
						\theta\\
						w
					\end{bmatrix}
					-
					\begin{bmatrix}
						\sigma_{\theta}	& 0\\
						0 & \sigma_{w}
					\end{bmatrix}
					B(\theta,w)$
                    \\[0.5ex]
					where $B(\theta,w)$ is computed according to \eqref{eqn:em_saddle_gradient}.
		\ENDFOR
	\end{algorithmic}
\end{algorithm}

Note that the linear rate is determined by two parts: (i) the strongly convex
regularization parameter $\rho$, and (ii) the positive definiteness of
$\widehat{A}^T \widehat{C}^{-1} \widehat{A}$.
The second part could be
interpreted as transferring  strong concavity in dual variables via the
full-rank bi-linear coupling matrix $\widehat{A}$. For this reason, even if
the saddle-point problem~\eqref{eqn:em_mspbe_saddle_point} has only strong
concavity in dual variables (when $\rho=0$), 
the algorithm still enjoys a linear convergence rate. 

Moreover, even if $\rho>0$, it will be inefficient to solve 
problem~\eqref{eqn:em_mspbe_saddle_point} using primal-dual algorithms based on 
proximal mappings of the strongly convex and concave terms
\citep[e.g.,][]{chambolle2011first,balamurugan2016stochastic}.
The reason is that, in~\eqref{eqn:em_mspbe_saddle_point}, the strong concavity of the 
Lagrangian with respect to the dual lies in the quadratic
function $(1/2)\|w\|_{\widehat{C}}$, whose proximal mapping cannot be computed
efficiently.
In contrast, the PDBG algorithm only needs its gradients.

If we pre-compute and store $\widehat{A}$, $\widehat{b}$ and
$\widehat{C}$, which costs $O(nd^2)$ operations, then computing
the gradient operator $B(\theta,w)$ in~\eqref{eqn:em_saddle_gradient}
during each iteration of PDBG costs $O(d^2)$ operations.
Alternatively, if we do not want to store these $d\times d$ matrices
(especially if $d$ is large),
then we can compute $B(\theta,w)$ as finite sums on the fly. 
More specifically,
$B(\theta,w) = \frac{1}{n}\sum_{t=1}^n B_t(\theta, w)$,
where for each $t=1,\ldots,n$,
\begin{align}
  B_t(\theta,w) = \left[\begin{array}{cc}
      \rho\theta - A_t w \\ A_t\theta - b_t + C_t w 
  \end{array} \right].
  \label{eqn:Bt}
\end{align}
Since $A_t$, $b_t$ and $C_t$ are all rank-one matrices,
as given in~\eqref{eqn:def-At-bt-Ct},
computing each $B_t(\theta,w)$ only requires $O(d)$ operations. 
Therefore, computing $B(\theta,w)$ costs $O(nd)$ 
operations as it averages $B_t(\theta,w)$ over~$n$ samples.

\section{Stochastic Variance Reduction Methods}
\label{sec:algos}

If we replace $B(\theta,w)$ in Algorithm~\ref{algo:PDBG_saddle} (line 2) 
by the stochastic gradient $B_t(\theta,w)$ in~\eqref{eqn:Bt}, 
then we recover the GTD2 algorithm of \citet{sutton2009fast}, 
applied to a fixed dataset, possibly with \emph{multiple passes}. 
It has a low per-iteration cost but a slow, \emph{sublinear} convergence rate.
In this section, we provide two stochastic variance reduction methods 
and show they achieve fast linear convergence.

\begin{algorithm}[tb]
	\renewcommand{\algorithmicrequire}{\textbf{Inputs:}}
	\renewcommand{\algorithmicensure}{\textbf{Outputs:}}
	\caption{SVRG for Policy Evaluation}
	\label{algo:SVRG_saddle}
	\begin{algorithmic}[1]
		\REQUIRE initial point $\left(\theta,w\right)$, 
        step sizes $\{\sigma_\theta, \sigma_w\}$,
        number~of outer iterations $M$, and number of inner iterations $N$.
		\FOR{$m=1$ {\bfseries to} $M$}
			\STATE Initialize $(\tilde{\theta}, \tilde{w}) = \left(\theta,w\right)$ and compute $B(\tilde{\theta},\tilde{w})$.
			\FOR{$j=1$ {\bfseries to} $N$}
            \STATE Sample an index $t_j$ from $\left\{1,\cdots,n\right\}$ and do
            \STATE Compute $B_{t_j}(\theta,w)$ and $B_{t_j}(\tilde{\theta},\tilde{w})$.
            \STATE 
                \vspace{0.8ex}
          $
					\begin{bmatrix}
							\theta \\
							w
						\end{bmatrix}  \leftarrow \begin{bmatrix}
						\theta \\
						w
					\end{bmatrix}  -  \begin{bmatrix}
						\sigma_\theta & 0 \\
						0 & \sigma_w
					\end{bmatrix} 
					B_{t_j}(\theta,w,\tilde{\theta},\tilde{w})
          $ \\[0.5ex]
					where $B_{t_j}(\theta,w,\tilde{\theta},\tilde{w})$
                    is given in~\eqref{Equ:SVRG:ReducedVarGrad}.
			\ENDFOR
		\ENDFOR
	\end{algorithmic}
\end{algorithm}

\subsection{SVRG for policy evaluation}
Algorithm \ref{algo:SVRG_saddle} is adapted from the stochastic variance
reduction gradient (SVRG) method~\cite{johnson2013accelerating}.
It uses two layers of loops and maintains two sets of parameters
$(\tilde{\theta},\tilde{w})$ and $(\theta,w)$. In the outer loop, the algorithm
computes a full gradient $B(\tilde{\theta},\tilde{w})$ using $(\tilde{\theta},
\tilde{w})$, which takes $O(nd)$ operations. Afterwards, the
algorithm executes the inner loop, which randomly samples an index $t_j$ and 
updates $(\theta,w)$ using variance-reduced stochastic gradient: 
	\begin{align}
	B_{t_j}\!(\theta,w,\tilde{\theta},\tilde{w}) = B_{t_j}\!(\theta,w) + B(\tilde{\theta},\tilde{w}) - B_{t_j}\!(\tilde{\theta},\tilde{w}).
		\label{Equ:SVRG:ReducedVarGrad}
	\end{align}
Here, $B_{t_j}(\theta,w)$ contains the stochastic gradients at $(\theta,w)$ 
computed using the random sample with index~$t_j$, and
$B(\tilde{\theta},\tilde{w}) - B_{t_j}(\tilde{\theta},\tilde{w})$ is a term
used to reduce the variance in $B_{t_j}(\theta,w)$ while keeping
$B_{t_j}\!(\theta,w,\tilde{\theta},\tilde{w})$ 
an unbiased estimate of $B(\theta,w)$. 

Since $B(\tilde{\theta},\tilde{w})$ is computed once during each iteration of 
the outer loop with cost $O(nd)$ 
(as explained at the end of Section~\ref{Sec:BatchGradient}),
and each of the $N$ iterations of the inner loop cost $O(d)$ operations,
the total computational cost of for each outer loop is $O(nd+Nd)$. 
We will present the overall complexity analysis of 
Algorithm~\ref{algo:SVRG_saddle} in Section~\ref{sec:analysis-svrg-saga}.

\subsection{SAGA for policy evaluation}

\begin{algorithm}[tb]
	\renewcommand{\algorithmicrequire}{\textbf{Inputs:}}
	\renewcommand{\algorithmicensure}{\textbf{Outputs:}}
	\caption{SAGA for Policy Evaluation}
	\label{algo:saga_saddle}
	\begin{algorithmic}[1]
		\REQUIRE initial point $\left(\theta,w\right)$,  
    step sizes $\sigma_\theta$ and $\sigma_w$, and number of iterations $M$.
        \STATE Compute each $g_t=B_t(\theta,w)$ for $t=1,\ldots, n$.
		\STATE Compute $B = B(\theta,w)=\frac{1}{n}\sum_{t=1}^n g_t$.
		\FOR{$m=1$ {\bfseries to} $M$}
		\STATE Sample an index $t_m$ from $\left\{1,\cdots,n\right\}$.
				\STATE Compute $h_{t_m} = B_{t_m}(\theta,w)$.
				\STATE 
        \vspace{0.8ex}
        $\begin{bmatrix}
				\theta \\
				w
				\end{bmatrix} \leftarrow \begin{bmatrix}
				\theta \\
				w
				\end{bmatrix}  -  \begin{bmatrix}
				\sigma_\theta & 0 \\
				0 & \sigma_w
      \end{bmatrix}\left(B + h_{t_m} - g_{t_m}\right)$.\\[0.5ex]
      \STATE $B \leftarrow B +\frac{1}{n}(h_{t_m}- g_{t_m})$
				\STATE $g_{t_m} \leftarrow h_{t_m}$.		
		\ENDFOR
	\end{algorithmic}
\end{algorithm}

The second stochastic variance reduction method for policy evaluation
is adapted from SAGA~\cite{defazio2014saga};
see Algorithm \ref{algo:saga_saddle}.
It uses a single loop, and maintains a single set of parameters $(\theta,w)$. 
Algorithm~\ref{algo:saga_saddle} starts by first computing each component
gradients $g_t=B_t(\theta,w)$ at the initial point, and also form their average
$B=\sum_t^n g_t$.
At each iteration, the algorithm randomly picks an index $t_m\in\{1,\ldots,n\}$ 
and computes the stochastic gradient $h_{t_m}=B_{t_m}(\theta,w)$.
Then, it updates $(\theta,w)$ using a variance reduced stochastic gradient:
$B + h_{t_m} - g_{t_m}$,
where $g_{t_m}$ is the previously computed stochastic gradient using the 
$t_m$-th sample (associated with certain past values of $\theta$ and $w$).
Afterwards, it updates the batch gradient estimate $B$ as 
$B+\frac{1}{n}(h_{t_m}-g_{t_m})$ and replaces $g_{t_m}$ with $h_{t_m}$. 

As Algorithm~\ref{algo:saga_saddle} proceeds, different vectors $g_t$ are 
computed using different values of~$\theta$ and~$w$
(depending on when the index~$t$ was sampled).
So in general we need to store all vectors $g_t$, for $t=1,\ldots,n$, 
to facilitate individual updates, which will cost additional $O(nd)$ storage.
However, by exploiting the rank-one structure in~\eqref{eqn:def-At-bt-Ct},
we only need to store three scalars
$(\phi_t-\gamma_\phi')^T\theta$, $(\phi_t-\gamma_\phi')^T w$, and $\phi_t^Tw$, 
and form $g_{t_m}$ on the fly using $O(d)$ computation.
Overall, each iteration of SAGA costs $O(d)$ operations.

\subsection{Theoretical analyses of SVRG and SAGA}
\label{sec:analysis-svrg-saga}

In order to study the convergence properties of SVRG and SAGA for 
policy evaluation, we introduce a smoothness parameter $L_G$
based on the stochastic gradients $B_t(\theta,w)$. 
Let $\beta = \sigma_{w}/\sigma_{\theta}$ be the ratio between the
primal and dual step-sizes, and define a pair of weighted Euclidean norms 
\begin{align*}
  \Omega(\theta,w) &\defeq (\|\theta\|^2 + \beta^{-1} \|w\|^2)^{1/2}, \\
\Omega^{*}(\theta,w) &\defeq (\|\theta\|^2 + \beta \|w\|^2)^{1/2}.
\end{align*}
Note that $\Omega(\cdot,\cdot)$ upper bounds the error in optimizing $\theta$: $\Omega(\theta-\theta_{\star},w-w_{\star}) \ge \|\theta-\theta_{\star}\|$.  Therefore, any bound on $\Omega(\theta-\theta_{\star},w-w_{\star})$ applies automatically to $\|\theta-\theta_{\star}\|$.

Next, we define the parameter $L_G$ through its square:
	\begin{align}
		L_G^2
			&\defeq
				\sup_{\theta_1,w_1, \theta_2, w_2} \!\!
				\frac{
					\frac{1}{n}
					\sum_{t=1}^n
					\Omega^{*}
					\big(B_t(\theta_1,w_1)-B_t(\theta_2,w_2)\big)^2
				}
				{
					\Omega(\theta_1-\theta_2, w_1-w_2)^2
				}.
				\nn
	\end{align}
This definition is similar to the smoothness constant~$\bar{L}$ used in
\citet{balamurugan2016stochastic} except that we used the step-size ratio
$\beta$ rather than the strong convexity and concavity parameters
of the Lagrangian to define
$\Omega$ and $\Omega^{*}$.\footnote{Since our saddle-point problem is not
  necessarily strongly convex in~$\theta$ (when $\rho=0$), 
  we could not define $\Omega$ and $\Omega^{*}$ in the same way 
  as \citet{balamurugan2016stochastic}.}
Substituting the definition of $B_t(\theta,w)$ 
in~\eqref{eqn:Bt}, we have 
\begin{equation}\label{eqn:LG-def}
  L_G^2 = \biggl\|\frac{1}{n} \sum_{t=1}^n  G_t^T G_t\biggr\|,
  ~\mbox{where}~
		G_t
			\defeq
				\begin{bmatrix}
					\rho I &\!\!\! -\sqrt{\beta} A_t^T \\
					\sqrt{\beta} A_t &\!\!\! \beta C_t
				\end{bmatrix}.
\end{equation}


With the above definitions, 
we characterize the convergence of
$\Omega(\theta_m-\theta_\star, w_m-w_\star)$,
where $(\theta_{\star}, w_{\star})$ is the solution 
of~\eqref{eqn:em_mspbe_saddle_point}, and $(\theta_m,w_m)$ is the output of 
the algorithms after the $m$-th iteration.
For SVRG, it is the $m$-th \emph{outer} iteration
in Algorithm~\ref{algo:SVRG_saddle}.
The following two theorems are proved in 
Appendices~\ref{Appendix:Proof_SVRG} and~\ref{Appendix:Proof_SAGA}, 
respectively.

\begin{thm}[Convergence rate of SVRG]
\label{thm:svrg}
Suppose Assumption~\ref{Asmp:nonsingular} holds.
If we choose
$\sigma_\theta = \frac{\mu_\rho}{48\kappa(\widehat{C}) L_G^2}$,  
$\sigma_w=\frac{8L_\rho}{\lambda_{\min}(\widehat{C})}\sigma_\theta$,
$N=\frac{51 \kappa^2\!(\widehat{C}) L_G^2}{\mu^2_\rho}$, 
where $L_\rho$ and $\mu_\rho$ are defined in~\eqref{eqn:L-rho} 
and~\eqref{eqn:mu-rho},
then 
\[
  \E\bigl[\Omega(\theta_m-\theta_\star, w_m-w_\star)^2\bigr]
  \leq \Bigl(\frac{4}{5}\Bigr)^m
  \Omega(\theta_0-\theta_\star, w_0-w_\star)^2 .
\]
The overall computational cost for reaching
$\E\bigl[\Omega(\theta_m-\theta_\star, w_m-w_\star)\bigr]\leq\epsilon$
is upper bounded by
\begin{align}
O\biggl(
\biggl(n+\frac{\kappa(\widehat{C}) L_G^2}{\lambda_{\min}^2(\rho I + \widehat{A}^T\widehat{C}^{-1}\widehat{A})}\biggr)d\;
\log\Bigl(\frac{1}{\epsilon}\Bigr)
\biggr).
\label{Equ:Thm_SVRG:ComplexityBound}
\end{align}
\end{thm}

\begin{thm}[Convergence rate of SAGA]
\label{thm:SAGA}
Suppose Assumption~\ref{Asmp:nonsingular} holds.
If we choose 
$\sigma_\theta = \frac{\mu_\rho}{3\left(8\kappa^2(\widehat{C})L_G^2 + n \mu_\rho^2\right)}$
and $\sigma_w=\frac{8L_\rho}{\lambda_{\min}(\widehat{C})}\sigma_\theta$
in Algorithm~\ref{algo:saga_saddle},
then
\[
  \E\bigl[\Omega(\theta_m-\theta_\star, w_m-w_\star)^2\bigr]
  \leq 2 (1-\rho)^m
  \Omega(\theta_0-\theta_\star, w_0-w_\star)^2 ,
\]
where 
$\rho\geq\frac{\mu_\rho^2}{9\left(8\kappa^2(\widehat{C})L_G^2 + n\mu_\rho^2\right)}$.
The total cost to achieve 
$\E\bigl[\Omega(\theta_m-\theta_\star, w_m-w_\star)\bigr]\leq\epsilon$
has the same bound in~\eqref{Equ:Thm_SVRG:ComplexityBound}.
\end{thm}

Similar to our PDBG results in \eqref{Equ:Thm_PDBG:ComplexityBound}, 
both the SVRG and SAGA algorithms for policy evaluation enjoy
linear convergence even if there is no strong convexity in the 
saddle-point problem \eqref{eqn:em_mspbe_saddle_point} (i.e., when $\rho = 0$).
This is mainly due to the
positive definiteness of $\widehat{A}^T \widehat{C}^{-1} \widehat{A}$ when
$\widehat{C}$ is positive-definite and $\widehat{A}$ is full-rank. 
In contrast, the linear convergence of SVRG and SAGA in
\citet{balamurugan2016stochastic} requires 
the Lagrangian to be both strongly convex in $\theta$ and
strongly concave in $w$.

Moreover, in the policy evaluation problem,
the strong concavity with respect to the dual variable~$w$
comes from a weighted quadratic norm $(1/2)\|w\|_{\widehat{C}}$, which
does not admit an efficient proximal mapping as required by
the proximal versions of SVRG and SAGA in
\citet{balamurugan2016stochastic}.
Our algorithms only require computing the stochastic gradients of 
this function, which is easy to do due to its finite sum structure.

\citet{balamurugan2016stochastic} also proposed accelerated variants of 
SVRG and SAGA using the ``catalyst'' framework of 
\citet{LinMairalHarchaoui2015catalyst}.
Such extensions can be done similarly for the three algorithms presented
in this paper, and we omit the details due to space limit.


\section{Comparison of Different Algorithms}

This section compares the computation complexities of several representative policy-evaluation algorithms that minimize EM-MSPBE, as summarized in
Table~\ref{Tab:Complexity}. 

The upper part of the table lists algorithms
whose complexity is linear in feature dimension $d$, including the two new algorithms presented in the previous section.
We can also apply GTD2
to a finite dataset with samples drawn uniformly at random with replacement.
It costs $O(d)$ per iteration, but has a sublinear 
convergence rate regarding $\epsilon$.  
In practice, people may choose $\epsilon=\Omega(1/n)$ for generalization reasons (see, e.g., \citet{Lazaric10Finite}), leading to an $O(\kappa'nd)$ overall complexity for GTD2, where $\kappa'$ is a condition number related to the algorithm.  However, as verified by our experiments, the bounds in the table show that our SVRG/SAGA-based algorithms are much faster as their effective condition numbers vanish when $n$ becomes large.  TDC has a similar complexity to GTD2.

In the table, we list two different implementations of PDBG. PDBG-(I)
computes the gradients by averaging the stochastic gradients over the entire
dataset at each iteration, which costs $O(nd)$ operations; see discussions at the end of Section~\ref{Sec:BatchGradient}.
PDBG-(II) first pre-computes the matrices $\widehat{A}$, $\widehat{b}$ and 
$\widehat{C}$ using $O(nd^2)$ operations, then computes the batch
gradient at each iteration with $O(d^2)$ operations. 
If $d$ is very large (e.g., when $d \gg n$), then PDBG-(I) would have an advantage over PDBG-(II). 
The lower part of the table also includes LSTD, which has $O(nd^2)$ complexity if rank-one updates are used.

SVRG and SAGA are more efficient than the other algorithms, when either $d$ or $n$ is very large. 
In particular, they have a lower complexity than LSTD when 
 $d > (1 + \frac{\kappa(\widehat{C})\kappa_G^2}{n} ) \log\Bigl(\frac{1}{\epsilon}\Bigr)$,
This condition is easy to satisfy, when $n$ is very large.
On the other hand, SVRG and SAGA algorithms are more efficient than PDBG-(I) if $n$ is large, say 
$n > \kappa(\widehat{C})\kappa_G^2\big/
\bigl(\kappa(\widehat{C})\kappa - 1\bigr)$,
where $\kappa$ and $\kappa_G$ are described in the caption of
Table~\ref{Tab:Complexity}.

\begin{table}[t]
\centering
\small
\caption{Complexity of different policy evaluation algorithms. In the table, $d$ is feature dimension, $n$ is dataset size, $\kappa \defeq \kappa(\rho I + \widehat{A}^T \widehat{C}^{-1} \widehat{A})$; $\kappa_G \defeq {L_G}/{\lambda_{\min}(\rho I + \widehat{A}^T \widehat{C}^{-1} \widehat{A})}$; and $\kappa'$ is a condition number related to GTD2.\vspace{0.5ex}}
\label{Tab:Complexity}
\renewcommand{\arraystretch}{1.5}
\begin{tabular}{ |c|l|}
\hline
{\bf Algorithm} & {\bf Total Complexity} \\
\hline
SVRG / SAGA & $O\left( nd \cdot  \Big(1+\frac{\kappa(\widehat{C}) \kappa_G^2}{n} \Big) \cdot  \log\big(1/\epsilon\big) \right)$  \\
GTD2 & $O\left( d \cdot \kappa'/\epsilon \right)$ \\
PDBG-(I) & $O\left( nd \cdot \kappa(\widehat{C})  \kappa \cdot \log(1/\epsilon) \right)$ \\
\hline
PDBG-(II) & $O\left( nd^2  + d^2 \kappa(\widehat{C})\kappa \cdot \log(1/\epsilon) \right)$ \\
LSTD & $O\bigl(nd^2\bigr)$ or $O\bigl(nd^2+d^3\bigr)$ \\
\hline
\end{tabular}
\end{table}

There are other algorithms whose complexity scales linearly with $n$ and $d$, including iLSTD~\cite{Geramifard07Ilstd}, and TDC~\cite{sutton2009fast}, fLSTD-SA~\cite{prashanth2014fast},
and the more recent algorithms of \citet{WangLiuFang2016NIPS} and 
\citet{DaiHePanBootsSong2016}.
However, their convergence is slow: 
the number of iterations required to reach a desired accuracy~$\epsilon$ 
grows as $1/\epsilon$ or worse.
The CTD algorithm~\cite{Korda15Td} uses a similar idea as SVRG to reduce variance in TD updates.  This algorithm is shown to have a similar linear convergence rate in an \emph{online} setting where the data stream is generated by a Markov process with \emph{finite} states and \emph{exponential} mixing.  The method solves for a fixed-point solution by stochastic approximation.  As a result, they can be non-convergent in off-policy learning, while our algorithms remain stable (c.f., Section~\ref{Sec:ExtensionOffPolicy}).


%


\section{Extensions}
\label{Sec:ExtensionOffPolicyEligibility}

It is possible to extend our approach to accelerate optimization of other objectives such as MSBE and NEU~\cite{dann2014policy}.  In this section, we briefly describe two extensions of the algorithms developed earlier.


\subsection{Off-policy learning}
\label{Sec:ExtensionOffPolicy}

In some cases, we may want to estimate the value function of a
policy $\pi$ from a set of data $\mathcal{D}$ generated 
by a different ``behavior'' policy $\pi_b$. 
This is called \emph{off-policy learning}~\citep[Chapter~8]{sutton1998reinforcement}.

In the off-policy case, samples are generated from the distribution induced by the behavior policy $\pi_b$, not the the target policy $\pi$.  While such a mismatch often causes stochastic-approximation-based methods to diverge~\cite{Tsitsiklis97Analysis}, our gradient-based algorithms remain convergent with the same (fast) convergence rate.

Consider the RL framework outlined in Section~\ref{sec:pre}.  For each state-action pair $(s_t,a_t)$ such that $\pi_b(a_t|s_t) > 0$, we define the importance ratio, $\rho_t \defeq \pi(a_t|s_t)/\pi_b(a_t|s_t)$.  The EM-MSPBE for off-policy learning has the same expression as in~\eqref{eqn:em-mspbe} except that $A_t$, $b_t$ and $C_t$ are modified by the weight factor~$\rho_t$, as listed in Table~\ref{Tab:ExpressionABC}; see also \citet[Eqn~6]{liu2015finite} for a related discussion.)  Algorithms~\ref{algo:PDBG_saddle}--\ref{algo:saga_saddle} remain the same for the off-policy case after $A_t$, $b_t$ and $C_t$ are modified correspondingly.

\subsection{Learning with eligibility traces}

Eligibility traces are a useful technique to trade off bias and variance in TD learning~\cite{Singh96Reinforcement,Kearns00Bias}.
When they are used, we can pre-compute $z_t$ in Table~\ref{Tab:ExpressionABC} before running our new algorithms. 
Note that EM-MSPBE with eligibility traces has the same form of
\eqref{eqn:em-mspbe}, with $A_t$, $b_t$ and $C_t$ defined differently according to the last row of Table~\ref{Tab:ExpressionABC}. 
At the $m$-th step of the learning process, the algorithm randomly samples $z_{t_m}, \phi_{t_m}, \phi_{t_m}'$ and $r_{t_m}$ from the fixed dataset and computes the corresponding stochastic gradients, where the index $t_m$ is uniformly distributed over $\{1,\ldots,n\}$ and are independent for different values of $m$. 
Algorithms~\ref{algo:PDBG_saddle}--\ref{algo:saga_saddle}
immediately work for this case, enjoying a similar linear convergence rate and a computation complexity linear in $n$ and $d$.
We need  additional $O(nd)$ operations to pre-compute $z_t$ recursively and an additional $O(nd)$ storage for $z_t$. 
However, it does not change the order of the total complexity for SVRG/SAGA. 


\begin{table}[t]
\centering
\small
\renewcommand{\arraystretch}{1.3}
\caption{Expressions of $A_t$, $b_t$ and $C_t$ for different cases of policy evaluation.  Here, $\rho_t \defeq \pi(a_t|s_t)/\pi_b(a_t | s_t)$; and $z_t \defeq \sum_{i=1}^t (\lambda \gamma)^{t-i} \phi_i$, where $\lambda \ge 0$ is a given parameter.\vspace{0.5ex}}
\label{Tab:ExpressionABC}
\begin{tabular}{ |c|c|c|c|}
\hline
& $A_t$ & $b_t$ & $C_t$ \\
\hline
{On-policy} & $\phi_t(\phi_t-\gamma\phi_t')^\top$ & $r_t \phi_t$ & $\phi_t \phi_t^\top$ \\
\hline
{Off-policy} & $\rho_t \phi_t (\phi_t-\gamma\phi_t')^\top$ & $\rho_t r_t \phi_t$ & $\phi_t\phi_t^\top$ \\
\hline
{Eligibility trace} & $z_t  (\phi_t-\gamma\phi_t')^\top$ & $r_t z_t$ & $\phi_t\phi_t^\top$ \\
\hline
\end{tabular}
\end{table}

\section{Experiments}
\label{sec:exp}


\begin{figure*}[t!]
	\centering
	\begin{subfigure}[t]{0.29\textwidth}
		\includegraphics[width=\textwidth]{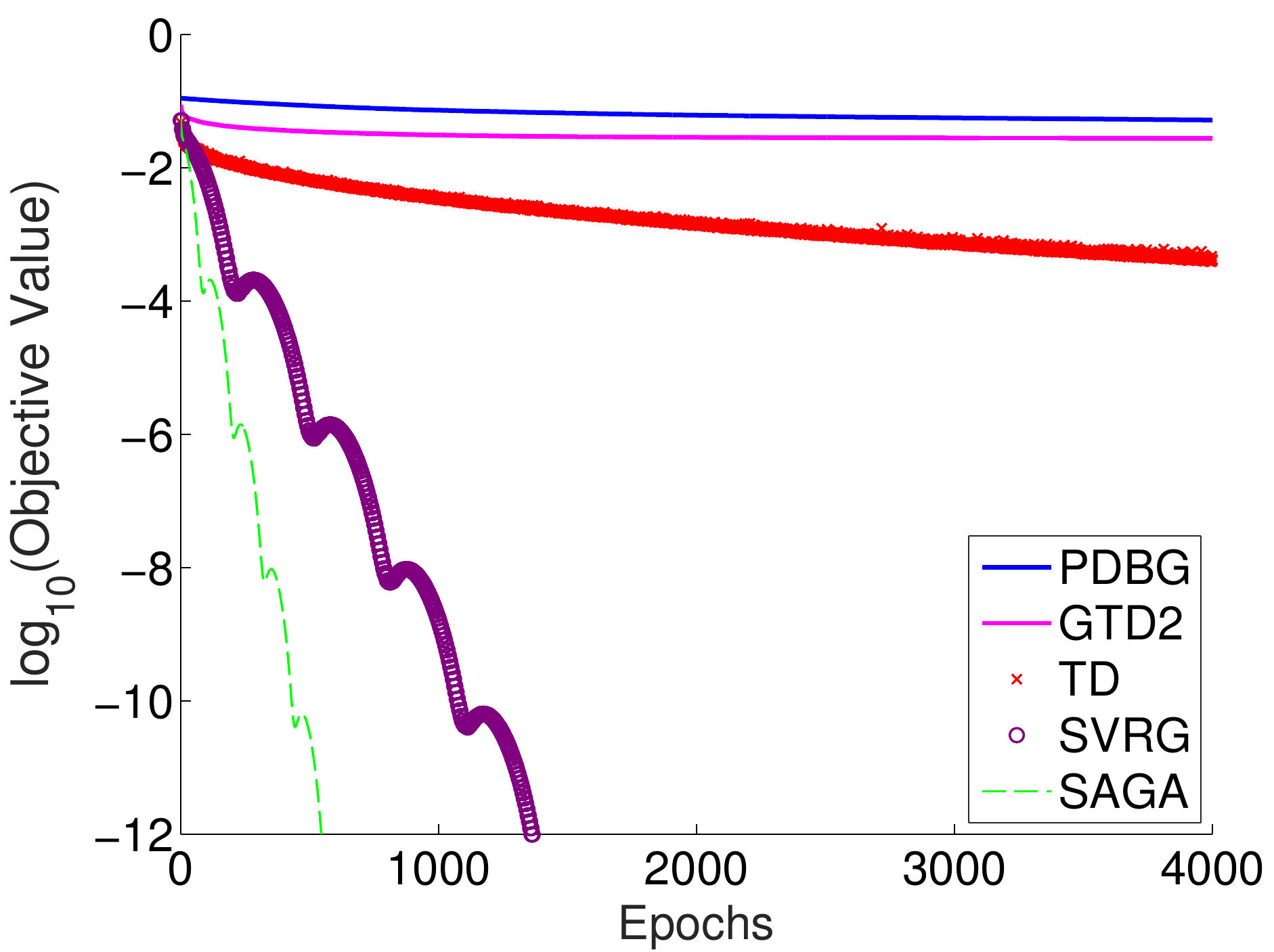}
		\caption{$\rho=0$}
	\end{subfigure}	
	\quad
	\begin{subfigure}[t]{0.29\textwidth}
		\includegraphics[width=\textwidth]{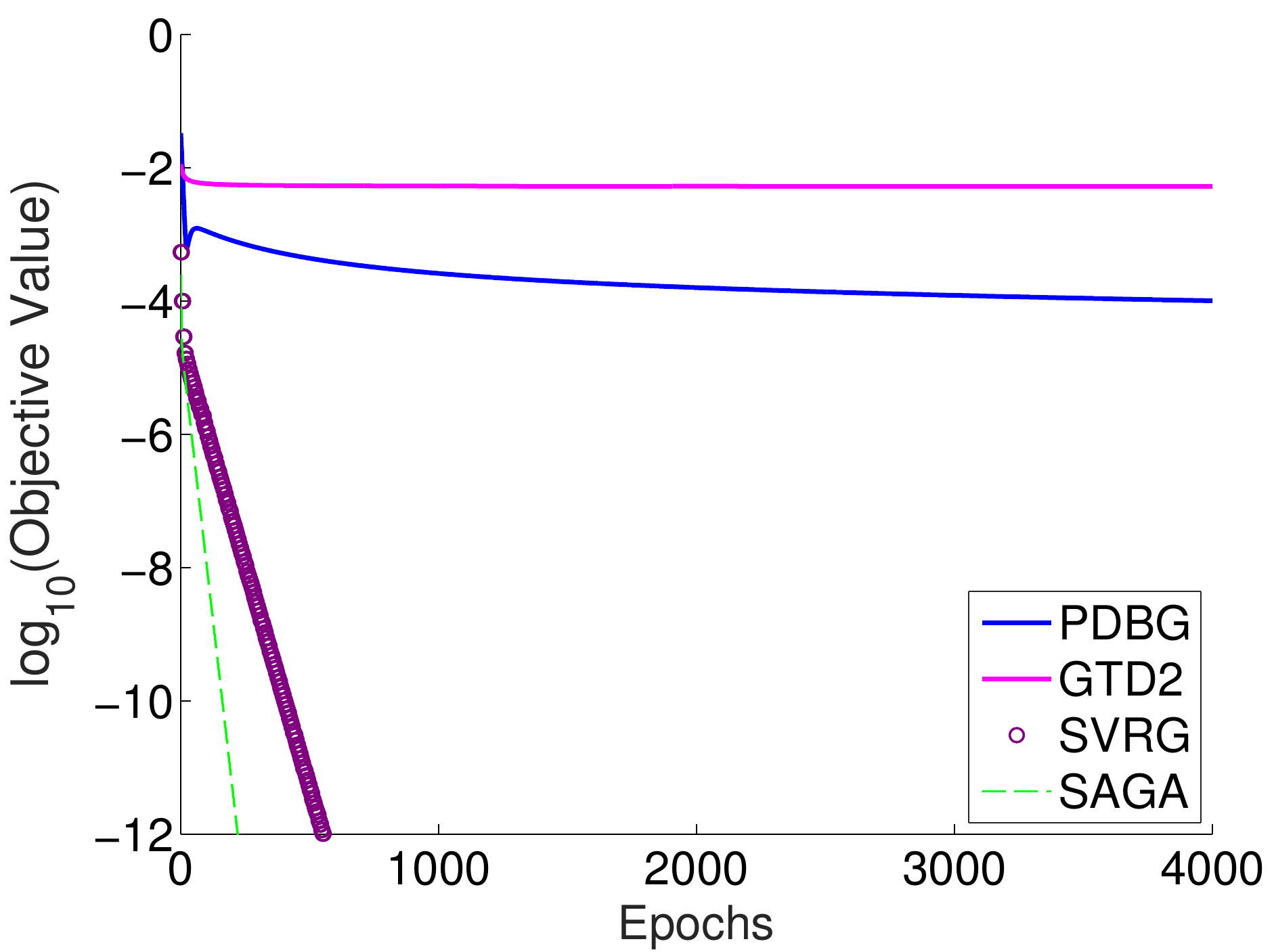}
		\caption{$\rho=\sqrt{\lambda_{\max}(\widehat{A}^\top \widehat{C}^{-1}\widehat{A})}$}
	\end{subfigure}
	\quad
	\begin{subfigure}[t]{0.29\textwidth}
		\includegraphics[width=\textwidth]{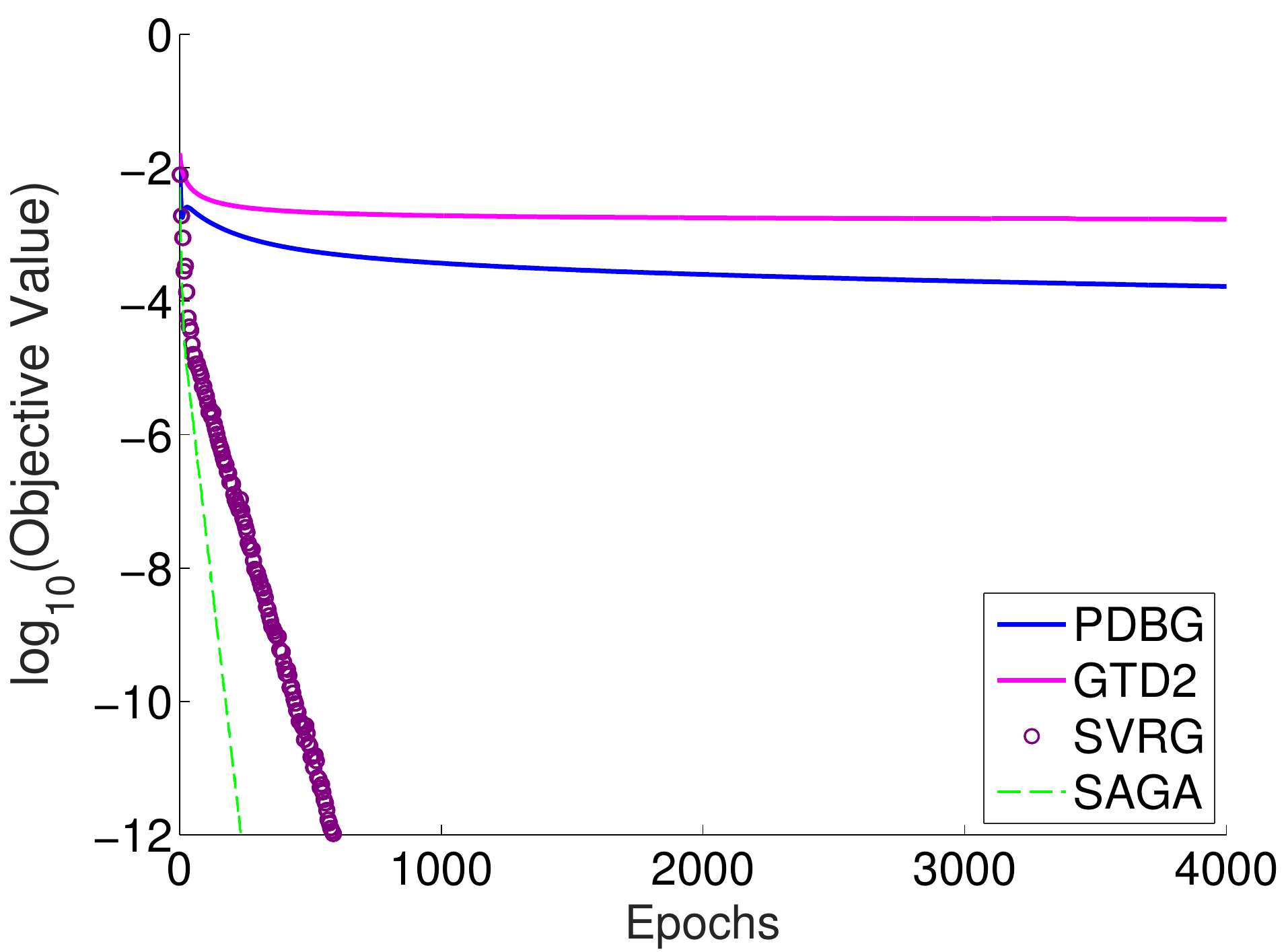}
		\caption{$\rho=\lambda_{\max}\left(\widehat{A}^\top \widehat{C}^{-1}\widehat{A}\right)$}
	\end{subfigure}
	\caption{Random MDP with $s=400$, $d=200$, and $n=20000$.
	}
	\label{fig:randmdp200d20000n}
\end{figure*}
\begin{figure*}[t!]
	\centering
	\begin{subfigure}[t]{0.29\textwidth}
		\includegraphics[width=\textwidth]{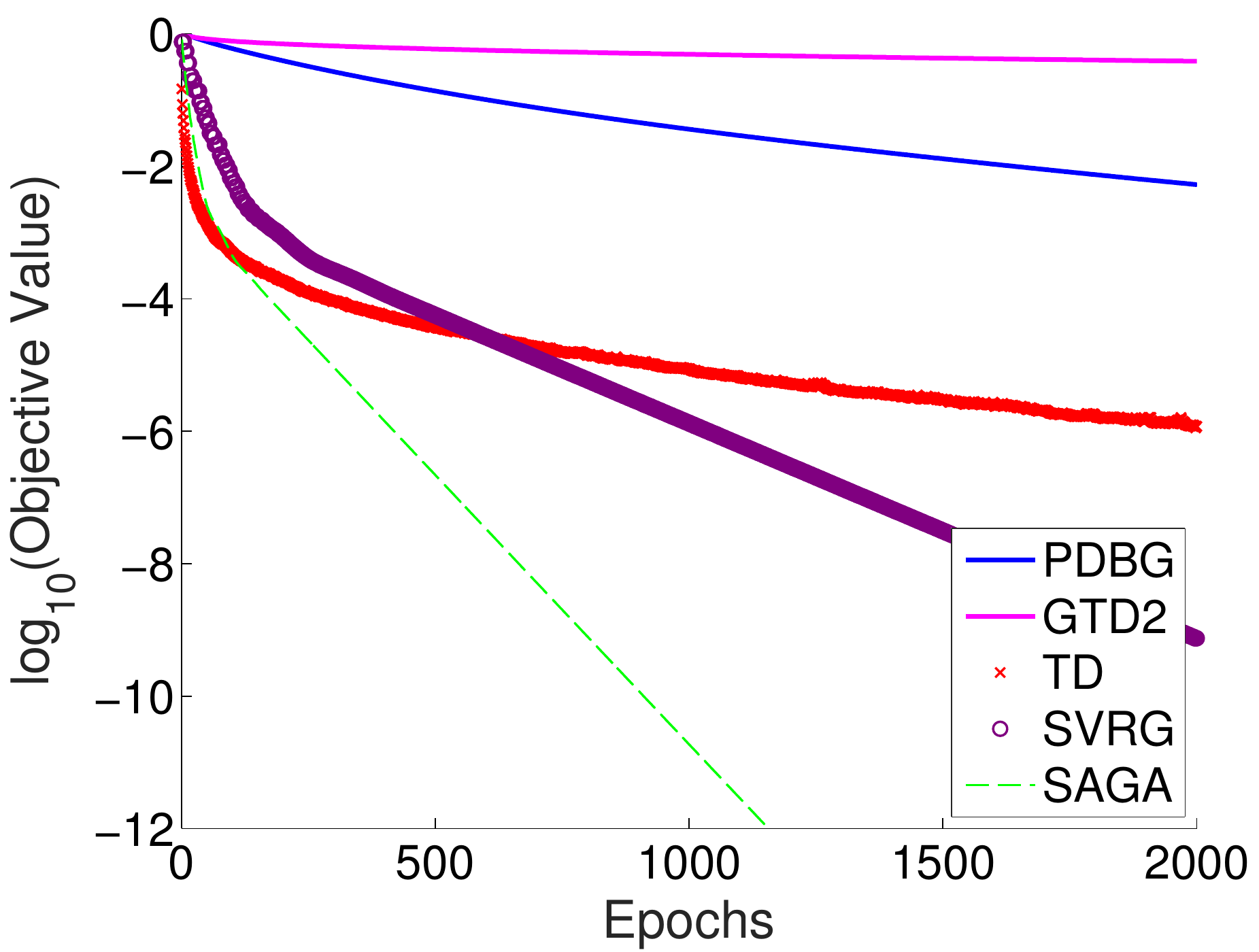}
		\caption{$\rho=0$}
	\end{subfigure}	
	\quad
	\begin{subfigure}[t]{0.29\textwidth}
		\includegraphics[width=\textwidth]{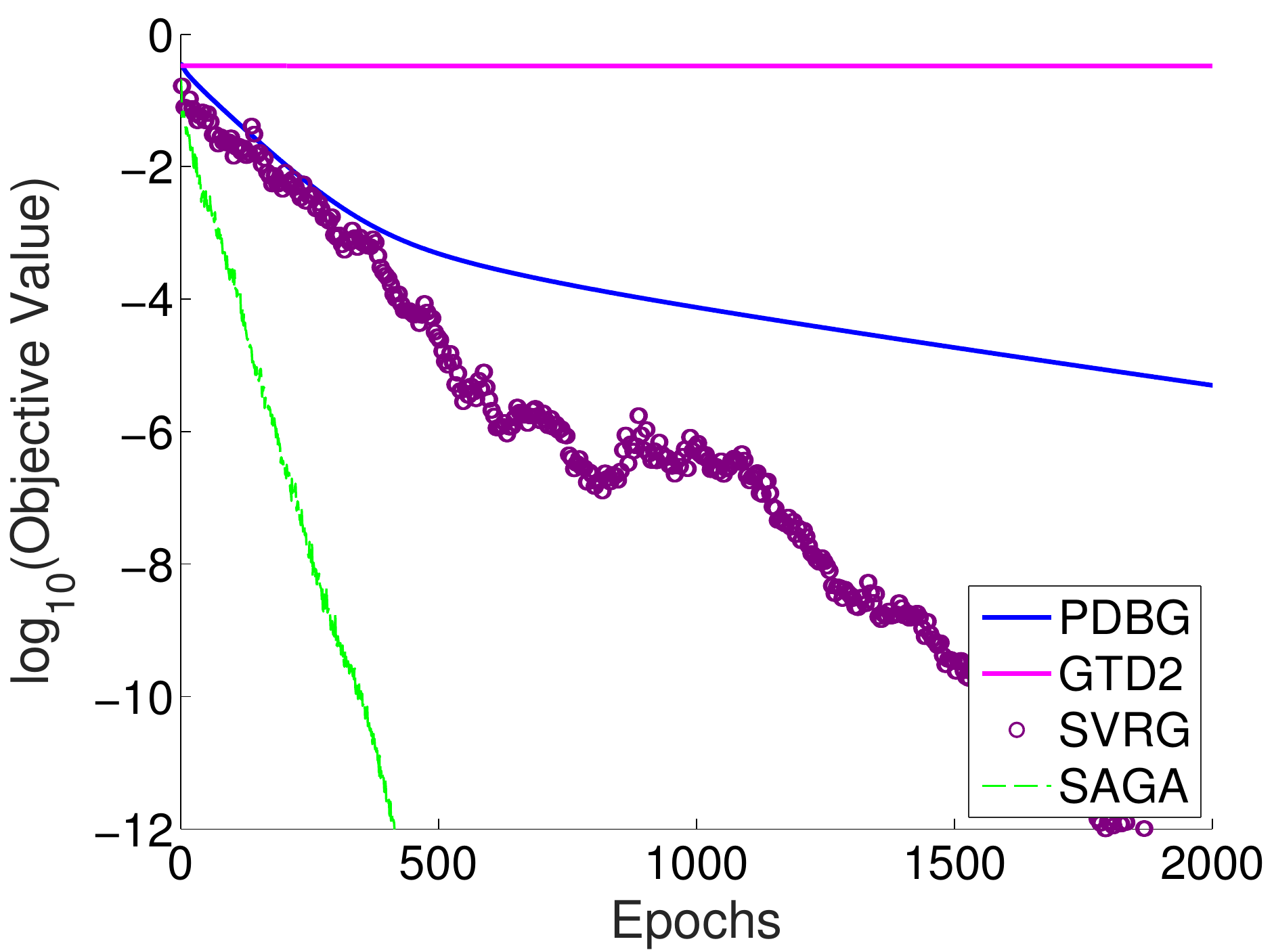}
		\caption{$\rho=0.01\lambda_{\max}(\widehat{A}^\top \widehat{C}^{-1}\widehat{A})$}
	\end{subfigure}
	\quad	
	\begin{subfigure}[t]{0.29\textwidth}
		\includegraphics[width=\textwidth]{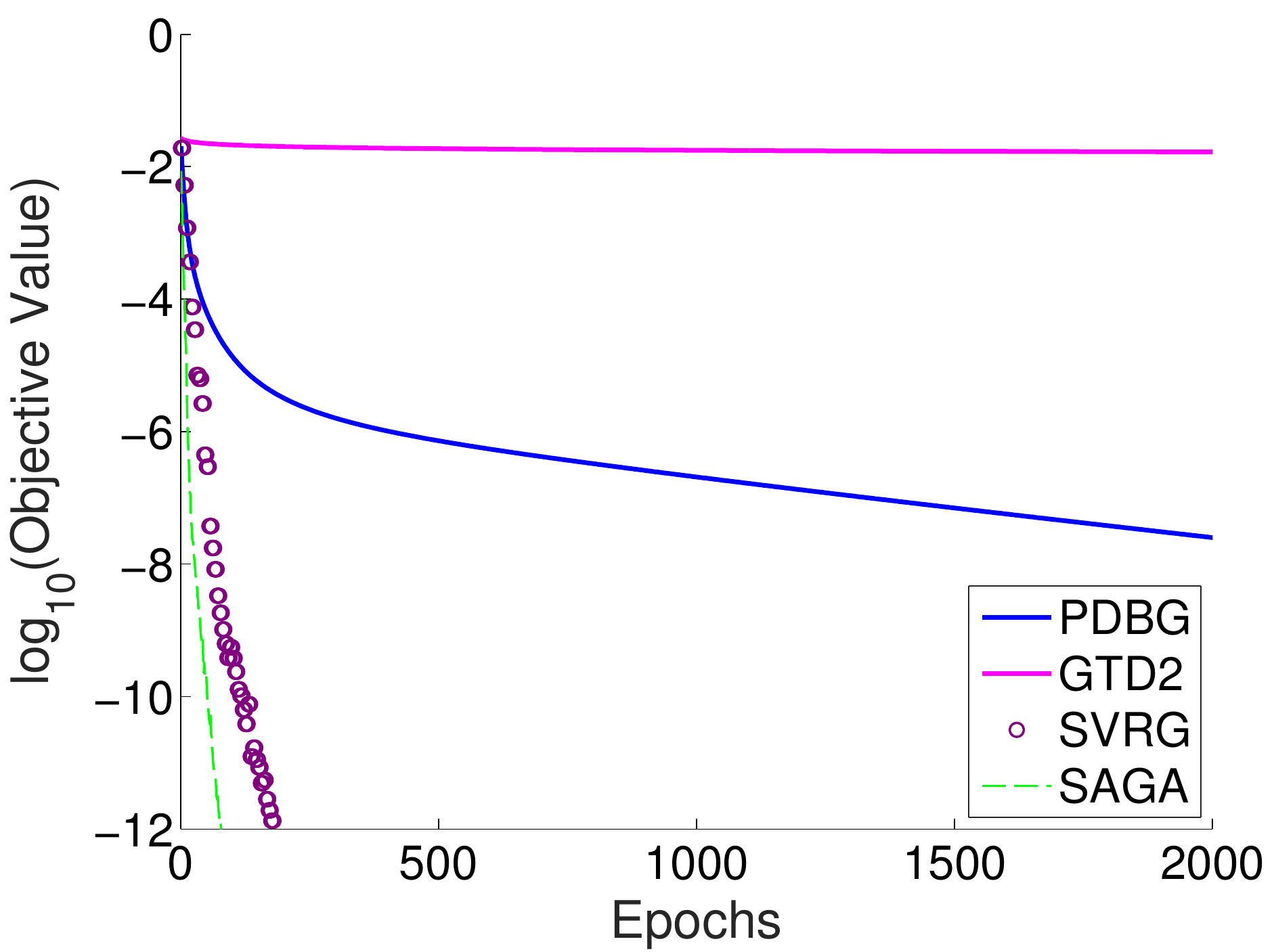}
		\caption{$\rho=\lambda_{\max}(\widehat{A}^\top \widehat{C}^{-1}\widehat{A})$}
	\end{subfigure}	
	\caption{Mountain Car Data Set with $d= 300$ and $n=5000$.
	}
	\label{fig:mountain_car_n5000}
\end{figure*}
\begin{figure*}[t!]
	\centering
	\begin{subfigure}[t]{0.29\textwidth}
		\includegraphics[width=\textwidth]{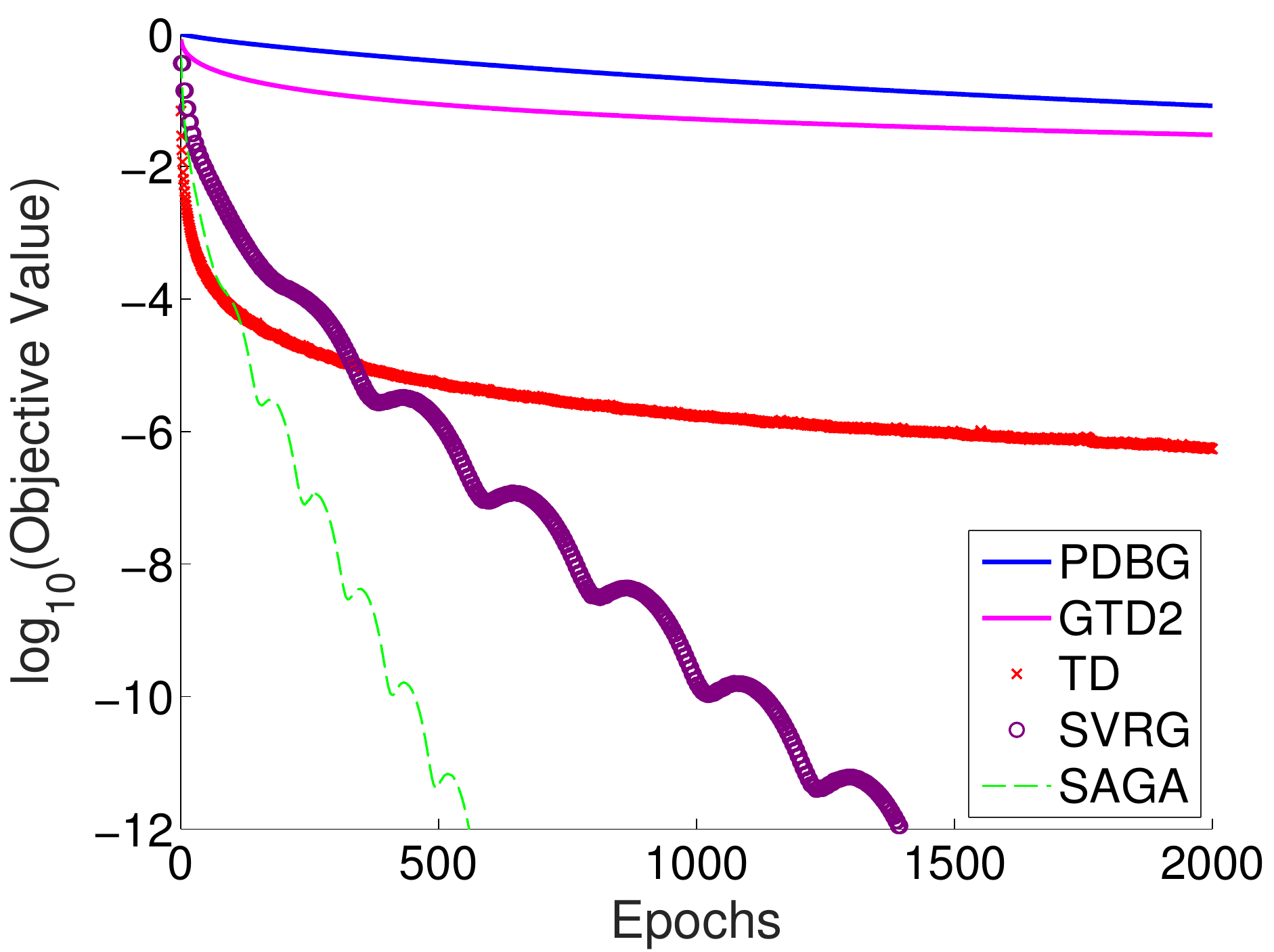}
		\caption{$\rho=0$}
	\end{subfigure}	
	\quad
	\begin{subfigure}[t]{0.29\textwidth}
		\includegraphics[width=\textwidth]{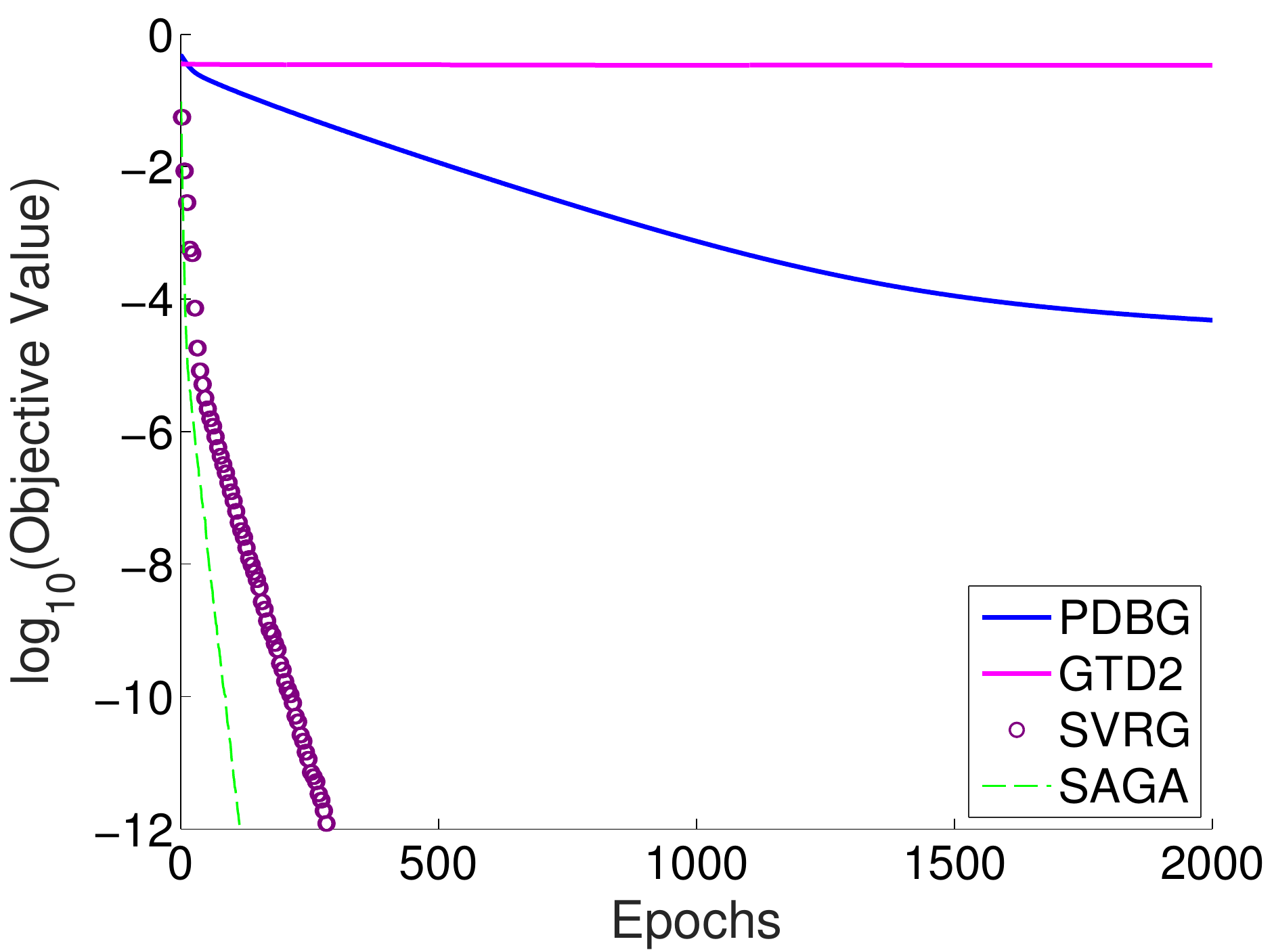}
		\caption{$\rho=0.01\lambda_{\max}(\widehat{A}^\top \widehat{C}^{-1}\widehat{A})$}
	\end{subfigure}
	\quad	
	\begin{subfigure}[t]{0.29\textwidth}
		\includegraphics[width=\textwidth]{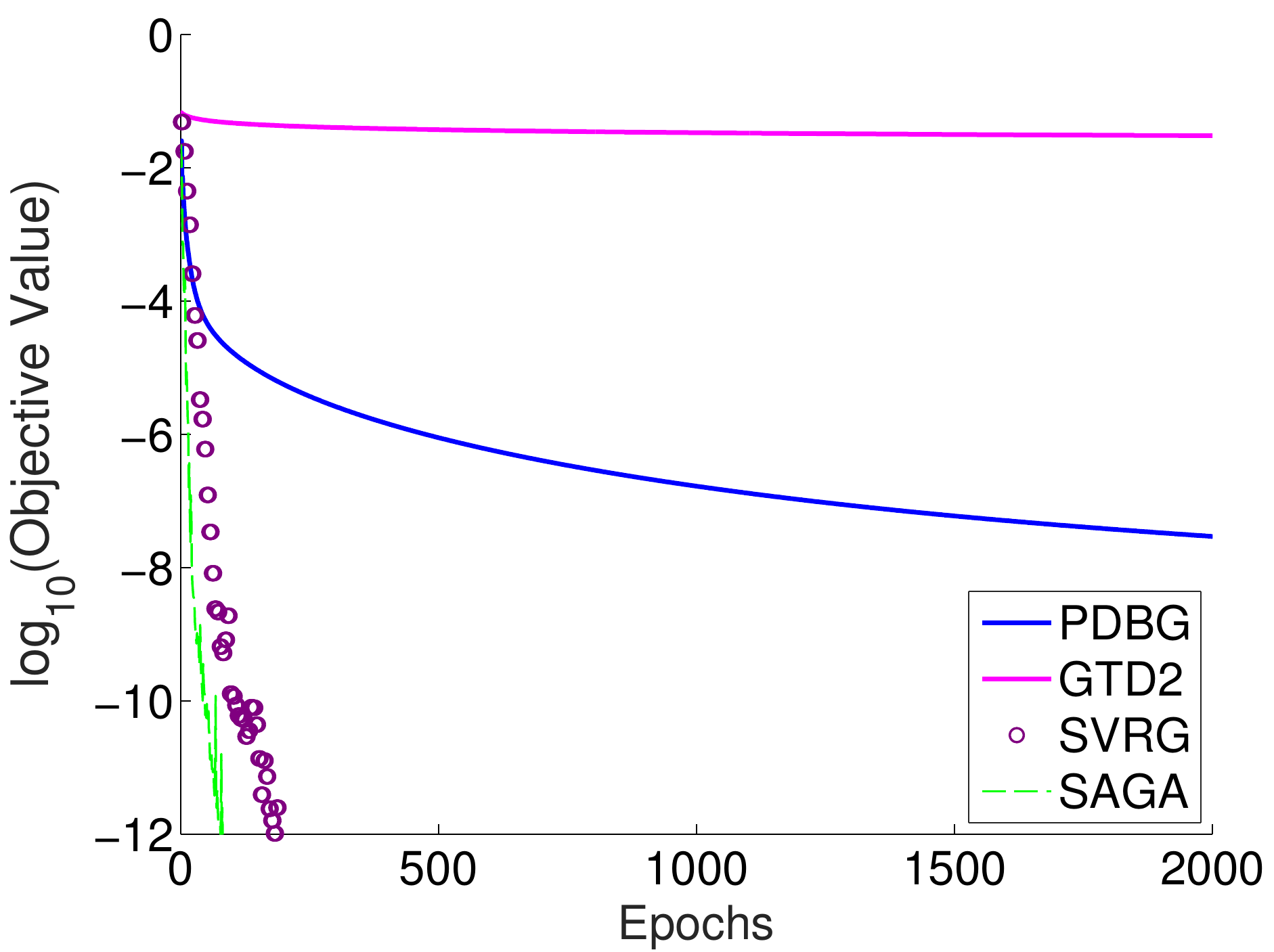}
		\caption{$\rho=\lambda_{\max}(\widehat{A}^\top \widehat{C}^{-1}\widehat{A})$}
	\end{subfigure}	
	\caption{Mountain Car Data Set with $d= 300$ and $n=20000$.
	}
	\label{fig:mountain_car_n20000}
\vspace{-0.5\baselineskip}
\end{figure*}


In this section, we compare the following algorithms on two benchmark problems: 
(i) \textbf{PDBG} (Algorithm~\ref{algo:PDBG_saddle}); 
(ii) \textbf{GTD2} with samples drawn randomly with replacement from a dataset;
(iii) \textbf{TD}: the fLSTD-SA algorithm of \citet{prashanth2014fast};
(iv) \textbf{SVRG} (Algorithm~\ref{algo:SVRG_saddle});
and (v) \textbf{SAGA} (Algorithm~\ref{algo:saga_saddle}).
Note that when $\rho > 0$, the TD solution and EM-MSPBE minimizer
differ, so we do not include \textbf{TD}.
%
For step size tuning, $\sigma_\theta$ is chosen from $\left\{10^{-1},10^{-2},\ldots,10^{-6}\right\}\frac{1}{L_\rho \kappa(\widehat{C})}$ and $\sigma_w$ is chosen from  $\left\{1,10^{-1},10^{-2}\right\}\frac{1}{\lambda_{\max}(\widehat{C})}$. 
We only report the results of each algorithm which correspond to the best-tuned step sizes; for SVRG we choose $N=2n$.

In the first task, we consider a randomly generated MDP with $400$ states and $10$ actions~\cite{dann2014policy}. 
The transition probabilities are defined as $P\left(s'|a,s\right) \propto p_{ss'}^a + 10^{-5}$, where $p_{ss'}^a \sim U[0,1]$.
The data-generating policy and start distribution were generated in a similar way. 
Each state is represented by a $201$-dimensional feature vector, where $200$ of the features were sampled from a uniform distribution, and the last feature was constant one.
We chose $\gamma=0.95$.
Fig.~\ref{fig:randmdp200d20000n} shows the performance of various algorithms for $n=20000$.
First, notice that the stochastic variance methods converge much faster than others.  
In fact, our proposed methods achieve linear convergence. 
Second, as we increase $\rho$, the performances of PDBG, SVRG and SAGA improve significantly due to better conditioning, as predicted by our theoretical results.
%

Next, we test these algorithms on Mountain Car~\citep[Chapter~8]{sutton1998reinforcement}.  To collect the dataset, we first ran Sarsa with $d=300$ CMAC features to obtain a good policy.  Then, we ran this policy to collect trajectories that comprise the dataset.
%
Figs.~\ref{fig:mountain_car_n5000} and \ref{fig:mountain_car_n20000} show our proposed stochastic variance reduction methods dominate other first-order methods.  Moreover, with better conditioning (through a larger $\rho$), PDBG, SVRG and SAGA achieve faster convergence rate.
Finally, as we increase sample size $n$, SVRG and SAGA converge faster.
This simulation verifies our theoretical finding in Table~\ref{Tab:Complexity} that SVRG/SAGA need fewer epochs for large $n$.

\section{Conclusions}
\label{sec:con}

In this paper, we reformulated the EM-MSPBE minimization problem in policy evaluation into an empirical saddle-point problem, and developed and analyzed a batch gradient method and two first-order stochastic variance reduction methods to solve the problem. 
An important result we obtained is that even when the reformulated saddle-point problem lacks strong convexity in primal variables and has only strong concavity in dual variables, the proposed algorithms are still able to achieve a linear convergence rate.  We are not aware of any similar results for primal-dual batch gradient methods or stochastic variance reduction methods. Furthermore, we showed that when both the feature dimension $d$ and the number of samples $n$ are large, the developed stochastic variance reduction methods are more efficient than any other gradient-based methods which are convergent in off-policy settings. 

This work leads to several interesting directions for research.  First, we believe it is important to extend the stochastic variance reduction methods to nonlinear approximation paradigms~\cite{bhatnagar2009convergent}, especially with deep neural networks.
Moreover, it remains an important open problem how to apply stochastic variance reduction techniques to policy optimization.



\bibliography{simonduref}
\bibliographystyle{icml2017}

\clearpage

\appendix

\section{Eigen-analysis of $\mat{G}$}
\label{sec:qp_analysis_G}

In this section, we give a thorough analysis of the spectral properties of the matrix
	\begin{align}
		G
			&=
				\begin{bmatrix}
					\rho I & - \beta^{1/2} \widehat{A}^T \\
					\beta^{1/2} \widehat{A} & \beta \widehat{C}
				\end{bmatrix},
		\label{Equ:Appendix:G_def}
	\end{align}
which is critical in analyzing the convergence of the PDBG, SAGA and SVRG 
algorithms for policy evaluation. 
Here $\beta=\sigma_w/\sigma_{\theta}$ is the ratio between the dual and primal
step sizes in these algorithms.
For convenience, we use the following notation:
	\begin{align}
		L
			&\triangleq
				\lambda_{\max}(\widehat{A}^T \widehat{C}^{-1} \widehat{A}),
				\nn\\
		\mu 
			&\triangleq
				\lambda_{\min}(\widehat{A}^T \widehat{C}^{-1} \widehat{A}) .
				\nn
	\end{align}
Under Assumption~\ref{Asmp:nonsingular}, they are well defined and
we have $L\geq \mu>0$.

\subsection{Diagonalizability of $G$}
\label{Appendix:DiagonalizabilityOfG}

First, we examine the condition of $\beta$ that ensures the diagonalizability of the matrix $G$. We cite the following result from \cite{shen2008condition}.

\begin{lem}
Consider the matrix~$\mathcal{A}$ defined as
		\begin{align}
			\mathcal{A}
				&=
					\begin{bmatrix}
						A	&	-B^\top \\
						B	&	C
					\end{bmatrix},
			\label{Equ:Lemma:DiagPaper:calA_def}
		\end{align}
where $A \succeq 0$, $C \succ 0$, and $B$ is full rank.  
Let $\tau = \lambda_{\min}(C)$, $\delta = \lambda_{\max}(A)$ and
$\sigma=\lambda_{\max}(B^\top C^{-1} B)$. 
    If $\tau > \delta + 2 \sqrt{\tau \sigma}$
    holds, then $\mathcal{A}$ is diagonalizable with all its eigenvalues real
    and positive.  
\end{lem}

Applying this lemma to the matrix $G$ in \eqref{Equ:Appendix:G_def}, we have
	\begin{align}
		\tau
				&=
						\lambda_{\min}(\beta \widehat{C})
				=
						\beta \lambda_{\min}(\widehat{C}),
						\nn\\
		\delta
                &=      \lambda_{\max}(\rho I) =	\rho,
						\nn\\
		\sigma
				&=
						\lambda_{\max}
						\bigl(
							\beta^{1/2} \widehat{A}^\top 
							(\beta\widehat{C})^{-1} 
							\beta^{1/2} \widehat{A}
						\bigr) 
				=
						\lambda_{\max}
						(\widehat{A}^\top \widehat{C}^{-1} \widehat{A}).
						\nn
	\end{align}
The condition $\tau > \delta + 2 \sqrt{\tau \sigma}$ translates into
	\begin{align}
		\beta \lambda_{\min}(\widehat{C}) 
				>
					\rho
					+ 
					2 
					\sqrt{
						\beta \lambda_{\min}(\widehat{C})
						\lambda_{\max}
						(\widehat{A}^\top \widehat{C}^{-1} \widehat{A})
					},
					\nn
	\end{align}
which can be solved as
	\begin{align}
		\sqrt{\beta}
				>
						\frac{
							\sqrt{
								\lambda_{\max}
								(\! \widehat{A}^\top \widehat{C}^{-1} \widehat{A} \!)
							}
							\!+\!							
							\sqrt{\rho\!+\!
								\lambda_{\max}
								(\! \widehat{A}^\top \widehat{C}^{-1} \widehat{A} \!)
							}
						}
						{
							\sqrt{ \lambda_{\min}(\widehat{C})} 
						} .
						\nn
	\end{align}
In the rest of our discussion, we choose $\beta$ to be
	\begin{align}
		\beta
			&=
					\frac{
							8\Bigl(\rho + 
								\lambda_{\max}
								\bigl(\widehat{A}^\top \widehat{C}^{-1} \widehat{A}\bigr)
                            \Bigr)
						}
						{ 
							\lambda_{\min}(\widehat{C}) 
						} 
            = \frac{8(\rho+L)}{\lambda_{\min}(\widehat{C})}, 
		\label{Equ:Appendix:beta_Value}
	\end{align}
which satisfies the inequality above.

\subsection{Analysis of eigenvectors}

If the matrix~$G$ is diagonalizable, then it can be written as
	\begin{align*}
		\mat{G} = \mat{Q}\mat{\Lambda}\mat{Q}^{-1},
	\end{align*} 
where $\Lambda$ is a diagonal matrix whose diagonal entries are the 
eigenvalues of~$G$, and~$Q$ consists of it eigenvectors (each with unit norm)
as columns.
Our goal here is to bound $\kappa(Q)$, the condition number of the matrix $Q$.
Our analysis is inspired by \citet{liesen2008nonsymmetric}. 
The core is the following fundamental result from linear algebra.

\begin{thm}[Theorem 5.1.1 of~\citet{gohberg2006indefinite}] \label{thm:eigenvec_ortho}
	Suppose $\mat{G}$ is diagonalizable. 
	If $\mat{H}$ is a symmetric positive definite matrix and $\mat{H}\mat{G}$ is symmetric, then there exist a complete set of eigenvectors of $\mat{G}$,
    such that they are orthonormal with respect to the inner product induced by $\mat{H}$:
	\begin{align} 
		\mat{Q}^\top\mat{H}\mat{Q} = \mat{I}.
	\end{align}
\end{thm}
If $H$ satisfies the conditions in Theorem~\ref{thm:eigenvec_ortho}, 
then we have $\mat{H} = \mat{Q}^{-\top}\mat{Q}^{-1}$, 
which implies $\kappa(H)=\kappa^2(Q)$.
Therefore, in order to bound $\kappa(Q)$, we only need to find such an $\mat{H}$ and analyze its conditioning.
To this end, we consider the matrix of the following form:
\begin{align}
	H 
		= 	
			\begin{bmatrix}
				(\delta - \rho) \mat{I} 	& 	\sqrt{\beta}\widehat{\mat{A}}^\top \\
				\sqrt{\beta}\widehat{\mat{A}} 	& 	\beta\widehat{\mat{C}} - \delta I
			\end{bmatrix}.
        \label{eqn:H-matrix}
\end{align}
It is straightforward to check that $HG$ is a symmetric matrix. The following lemma  states the conditions for $\mat{H}$ being positive definite.
\begin{lem}
	If $\delta - \rho > 0$ and $\beta\widehat{\mat{C}} - \delta\mat{I} - \frac{\beta}{\delta - \rho}\widehat{\mat{A}}\widehat{\mat{A}}^\top \succ \mat{0}$, then $\mat{H}$ is positive definite.
\end{lem}
\begin{proof}
The matrix~$H$ in~\eqref{eqn:H-matrix} admits the following
Schur decomposition:
\begin{align*}
	\mat{H} 
		&= 
			\begin{bmatrix}
				\mat{I} & \mat{0} \\
				\frac{\sqrt{\beta}}{\delta-\rho}\widehat{\mat{A}}  & \mat{I}
			\end{bmatrix}
			\begin{bmatrix}
				(\delta-\rho)I & \\
				& \mat{S}
			\end{bmatrix}
			\begin{bmatrix}
				\mat{I} & \frac{\sqrt{\beta}}{\delta-\rho}\widehat{\mat{A}}^\top \\
				\mat{0} & \mat{I}
			\end{bmatrix},
\end{align*}
where  $\mat{S} = \beta\widehat{\mat{C}} - \delta\mat{I} - \frac{\beta}{\delta - \rho}\widehat{\mat{A}}\widehat{\mat{A}}^\top$.
Thus $\mat{H}$ is congruence to the block diagonal matrix in the middle, 
which is positive definite under the specified conditions. 
Therefore, the matrix $\mat{H}$ is positive definite under the same conditions.
\end{proof}

In addition to the choice of $\beta$ in \eqref{Equ:Appendix:beta_Value}, we choose $\delta$ to be
	\begin{align}
		\delta
			&=
					4(\rho + L) .
		\label{Equ:Appendix:delta_value}
	\end{align}
It is not hard to verify that this choice ensures $\delta-\rho>0$ and
$\beta\widehat{\mat{C}} - \delta\mat{I} - \frac{\beta}{\delta -
\rho}\widehat{\mat{A}}\widehat{\mat{A}}^\top \succ \mat{0}$ so that $H$ is
positive definite. We now derive an upper bound on the condition number of $H$.
Let $\lambda$ be an eigenvalue of $H$ and $[x^T y^T]^T$ be its associated
eigenvector, where $\|x\|^2+\|y\|^2 > 0$. Then it holds that
	\begin{align}
		(\delta-\rho) x + \sqrt{\beta}\widehat{A}^T y  &= \lambda x,
	\label{Equ:Appendix:EigenH1}
		\\
		\sqrt{\beta}\widehat{A} x + (\beta\widehat{C} - \delta I ) y &= \lambda y.
	\label{Equ:Appendix:EigenH2}
	\end{align}
From \eqref{Equ:Appendix:EigenH1}, we have
	\begin{align}
		x
			&=
					\frac{\sqrt{\beta}}{\lambda-\delta+\rho} \widehat{A}^T y .
		\label{Equ:Appendix:EigenH1_interm}
	\end{align}
Note that $\lambda-\delta+\rho \neq 0$ because if $\lambda-\delta+\rho = 0$ we have $\widehat{A}^T y = 0$ so that $y=0$ since $\widehat{A}$ is full rank. With $y=0$ in \eqref{Equ:Appendix:EigenH2}, we will have $\widehat{A} x = 0$ so that $x=0$, which contradicts the assumption that $\|x\|^2+\|y\|^2 > 0$. 

Substituting \eqref{Equ:Appendix:EigenH1_interm} into \eqref{Equ:Appendix:EigenH2} and multiplying both sides with $y^T$, we obtain the following equation after some algebra
	\begin{align}
		\lambda^2 - p \lambda + q = 0,
		\label{Equ:Appendix:EigenH_QuadEquation}
	\end{align}
where
	\begin{align}
		p
			&\defeq
					\delta- \rho + \frac{y^T (\beta\widehat{C} - \delta I) y}{\|y\|^2},
					\nn\\
		q
			&\defeq
					(\delta-\rho) \frac{y^T (\beta\widehat{C} - \delta I) y}{\|y\|^2} 
					- 
					\beta\frac{y^T \widehat{A} \widehat{A}^T y}{\|y\|^2} .
					\nn
	\end{align}
We can verify that both $p$ and $q$ are positive with our choice of $\delta$ and $\beta$. The roots of the quadratic equation in~\eqref{Equ:Appendix:EigenH_QuadEquation} are given by
	\begin{align}
		\lambda
			&=
					\frac{p \pm \sqrt{p^2-4q}}{2} .
	\end{align}
Therefore, we can upper bound the largest eigenvalue as
	\begin{align}
		\lambda_{\max}(H)
			&\le
					\frac{p + \sqrt{p^2-4q}}{2}
          \nn\\
			& \le
					p=\delta- \rho - \delta + \beta \frac{y^T \widehat{C} y}{\|y\|^2}
					\nn\\
			&\le 
          -\rho + \beta \lambda_{\max}(\widehat{C})
					\nn\\
          &=		- \rho + \frac{8(\rho+L)}{\lambda_{\min}(\widehat{C})} \lambda_{\max}(\widehat{C})		
					\nn\\
			&\le
      8 (\rho+L) \kappa(\widehat{C}).
	\end{align}
Likewise, we can lower bound the smallest eigenvalue:
	\begin{align}
		\lambda_{\min}(H)
			&\ge
					\frac{p - \sqrt{p^2-4q}}{2}
			\ge
					\frac{p-p+2q/p}{2}
			=
					\frac{q}{p}
					\nn\\
			&=
					\frac{
						\beta
						\left(
							(\delta-\rho)
							\frac{y^T \widehat{C} y}{\|y\|^2}
							-
							\frac{y^T \widehat{A}\widehat{A}^Ty}{\|y\|^2}
						\right)
						-
						\delta (\delta - \rho)
					}
					{
						- \rho + \beta \frac{y^T \widehat{C} y}{\|y\|^2}
					}
					\nn\\
			&\overset{(a)}{\ge}
					\frac{
						\beta
						\left(
							(\delta-\rho)
							\frac{y^T \widehat{C} y}{\|y\|^2}
							-
							\frac{y^T \widehat{A}\widehat{A}^Ty}{\|y\|^2}
						\right)
						-
						\delta (\delta - \rho)
					}
					{
						\beta \frac{y^T \widehat{C} y}{\|y\|^2}
					}
					\nn\\
			&=
					\delta - \rho - \frac{y^T \widehat{A}\widehat{A}^T y}{y^T \widehat{C} y}
					-
					\frac{\delta (\delta-\rho)}{\beta}
					\cdot
					\frac{1}{\frac{y^T \widehat{C} y}{\|y\|^2}}
					\nn\\
			&\overset{(b)}{\ge}
					\delta - \rho - L
					-
          \frac{\delta (\delta-\rho)}{\beta \lambda_{\min}(\widehat{C})}
					\nn\\
			&\overset{(c)}{=}
					(\rho+L)
					\left(
						3 - \frac{3\rho+4L}{2(\rho+L)}
					\right)
					\nn\\
			&\ge
					\rho+L,	
	\end{align}
where step (a) uses the fact that both the numerator and denominator are
positive, step (b) uses the fact
\[
  L\triangleq\lambda_{\max}\Bigl(\widehat{A}^T\widehat{C}^{-1}\widehat{A}\Bigr)
  \geq  \frac{y^T\widehat{A}\widehat{A}^T y}{y^T \widehat{C}y},
\]
and step (c) substitutes the expressions of $\delta$ and $\beta$.  
Therefore, we can upper bound the condition number of $H$, 
and thus that of~$Q$, as follows:
	\begin{align}
		\kappa^2(Q)
			=
					\kappa(H)
			&\le
      \frac{8(\rho+L)\kappa(\widehat{C})}{\rho+L}
			=
					8\kappa(\widehat{C}) .
		\label{Equ:Appendix:CondNum_H}
	\end{align}

\subsection{Analysis of eigenvalues}
Suppose $\lambda$ is an eigenvalue of~$G$ and let 
$\left(\xi^\top, \eta^\top\right)^\top$ be its corresponding eigenvector.
By definition, we have 
\begin{align*}
	\mat{G}
	\begin{bmatrix}
	\xi \\
	\eta 
	\end{bmatrix} = \lambda \begin{bmatrix}
	\xi \\
	\eta 
	\end{bmatrix},
\end{align*} which is equivalent to the following two equations:
\begin{align*}
	\rho \xi -\sqrt{\beta}\widehat{\mat{A}}^\top \eta = \lambda \xi, \\
	\sqrt{\beta}\widehat{\mat{A}}\xi + \beta \widehat{\mat{C}}\eta = \lambda \eta.
\end{align*}
Solve $\xi$ in the first equation in terms of $\eta$, 
then plug into the second equation, we obtain:
\begin{align*}
	\lambda^2\eta - \lambda ( \rho \eta + \beta \widehat{\mat{C}}\eta) + \beta ( \widehat{\mat{A}}\widehat{\mat{A}}^\top\eta + \rho \widehat{C} \eta ) = 0.	
\end{align*}
Now left multiply $\eta^\top$, then divide by the $\norm{\eta}_2^2$, we have:
\begin{align*}
	\lambda^2 - p\lambda + q = 0.
\end{align*} 
where $p$ and $q$ are defined as
	\begin{align}
		p
			&\defeq
				\rho + \beta \frac{\eta^\top \widehat{C} \eta}{\| \eta \|^2},
				\nn\\
		q
			&\defeq
				\beta
				\left(
					\frac{\eta^T \widehat{A} \widehat{A}^\top \eta }{\| \eta \|^2}
					+
					\rho
					\frac{\eta^T \widehat{C} \eta}{\|\eta\|^2}
				\right).
	\end{align}
Therefore the eigenvalues of $\mat{G}$ satisfy:
\begin{align}
	\lambda = \frac{p \pm \sqrt{p^2 -4q}}{2}. \label{eqn:qp_eigenvalue_formula}
\end{align}
Recall that our choice of $\beta$ ensures that $G$ is diagonalizable and has positive real eigenvalues. Indeed, we can verify that the diagonalization condition guarantees $p^2 \ge 4q$ so that all eigenvalues are real and positive.
Now we can obtain upper and lower bounds based on~\eqref{eqn:qp_eigenvalue_formula}.
For upper bound, notice that 
\begin{align}
	\lambda_{\max}(G)
		&\le 
			p 
		\le 
    \rho + \beta \lambda_{\max}(\widehat{C})
			\nn\\
		&=
			\rho + \frac{8(\rho+L)}{\lambda_{\min}(\widehat{C}}\lambda_{\max}(\widehat{C})
			\nn\\
		&=
    \rho + 8 (\rho+L) \kappa(\widehat{C}) 
			\nn\\
		&\le
			9 \kappa(\widehat{C}) \bigl(\rho+L\bigr)
			\nn\\
		&=
    9 \kappa(\widehat{C}) \lambda_{\max}\bigl( \rho I + \widehat{A}^T \widehat{C}^{-1} \widehat{A}\bigr) .
			\label{eqn:qp_eigenvalue_upper}
\end{align}
For lower bound, notice that 
\begin{align}
	\lambda_{\min}(G)
		&\ge 
			\frac{p-\sqrt{p^2-4q}}{2} 
		\ge 
			\frac{p-p+2q/p}{2} 
		= 
			q/p 
			\nonumber \\
		&=
			\frac{
				\beta
				\Big(
					\frac{\eta^T \widehat{A}\widehat{A}^T \eta}{\eta^T \widehat{C} \eta}
					+ 
					\rho
				\Big)
			}
			{
				\rho 
				\frac{\|\eta\|^2}{\eta^T \widehat{C} \eta}
				+ 
				\beta
			}
			\nn\\
        &\overset{(a)}{\ge}
        \frac{\beta (\rho + \mu) }{\rho/\lambda_{\min}(\widehat{C}) + \beta}
		=
			\frac{\beta \lambda_{\min}(\widehat{C})(\rho+\mu)}{\rho + \beta \lambda_{\min}(\widehat{C})}
			\nn\\
        &\overset{(b)}{=}
			\frac{8(\rho+L)(\rho + \mu)}{\rho + 8(\rho+L)}
			\nn\\
		&\ge
			\frac{8}{9}
			(\rho+\mu)
			\nn\\
		&=
			\frac{8}{9}
			(\rho+\lambda_{\min}(\widehat{A}^T \widehat{C}^{-1} \widehat{A}))
			\nn\\
		&=
			\frac{8}{9}
			\lambda_{\min}(\rho I+\widehat{A}^T \widehat{C}^{-1} \widehat{A}),
	\label{Equ:Appendix:lambda_min_G_LB}
\end{align}
where the second inequality is by the concavity property of the square root function, step~(a) used the fact
\[
  \mu\triangleq\lambda_{\min}\Bigl(\widehat{A}^T\widehat{C}^{-1}\widehat{A}\Bigr)
  \leq  \frac{y^T\widehat{A}\widehat{A}^T y}{y^T \widehat{C}y},
\]
and step (b) substitutes the expressions of $\beta$.

Since ~$G$ is not a normal matrix, we cannot use their eigenvalue bounds
to bound its condition number $\kappa(G)$.

\section{Linear convergence of PDBG}
\label{sec:PDBG-analysis}

Recall the saddle-point problem we need to solve:
\[
  \min_\theta \max_w ~ \mathcal{L}(\theta,w) ,
\]
where the Lagrangian is defined as
\begin{align}
  \mathcal{L}(\theta,w) 
			&=  \frac{\rho}{2} \|\theta\|^2 - w^\top \widehat{\mat{A}} \theta 
					-\frac{1}{2}w^\top \widehat{\mat{C}} w + \widehat{b}^\top w.
	\label{eqn:qp_saddle}
\end{align}
Our assumption is that $\widehat{\mat{C}}$ is positive definite and $\widehat{A}$ has full rank. 
The optimal solution can be expressed as
\begin{align*}
	\theta_\star 
			&= 
				\left(\widehat{A}^\top \widehat{\mat{C}}^{-1}\widehat{A} 
				+ \rho I\right)^{-1}
                \widehat{\mat{A}}^\top \widehat{C}^{-1}	\widehat{b},
                \\
	w_\star 
			&= 
				\widehat{\mat{C}}^{-1}
				\left(
					\widehat{b} -
					\widehat{\mat{A}}^\top
					\theta_\star 
				\right) .
\end{align*}
The gradients of the Lagrangian with respect to~$\theta$ and~$w$, respectively,
are
\begin{align*}
  \nabla_\theta \mathcal{L}\left(\theta,w\right) 
	&= 
		 \rho\theta -\widehat{\mat{A}}^\top w \\
  \nabla_w \mathcal{L}\left(\theta, w\right) 
	&=  
		-\widehat{\mat{A}}\theta - \widehat{\mat{C}} w + \widehat{b}.
\end{align*}
The first-order optimality condition is obtained by setting them to zero,
which is satisfied by $(\theta_\star, w_\star)$:
	\begin{align}
			\begin{bmatrix}
				\rho I & -\widehat{A}^\top\\
				\widehat{\mat{A}} & \widehat{\mat{C}}
			\end{bmatrix}
			\begin{bmatrix}
				\theta_{\star} \\
				w_{\star}
			\end{bmatrix} 
			= 
			\begin{bmatrix}
				0\\
				\widehat{b}
			\end{bmatrix} .
			\label{eqn:qp_first_order_necessary}
	\end{align}

The PDBG method in Algorithm~\ref{algo:PDBG_saddle} takes the following
iteration:
\[
                    \begin{bmatrix}
                      \theta_{m+1}\\
                      w_{m+1}
					\end{bmatrix}
                    =
					\begin{bmatrix}
						\theta_m\\
						w_m
					\end{bmatrix}
					-
					\begin{bmatrix}
						\sigma_{\theta}	& 0\\
						0 & \sigma_{w}
					\end{bmatrix}
					B(\theta_m,w_m) ,
\]
where
\begin{align*}
B(\theta,w)
		&=
			\begin{bmatrix}
				\nabla_{\theta} L(\theta,w) \\
				-\nabla_{w} L(\theta,w)
			\end{bmatrix}
		=
			\begin{bmatrix}
				\rho I & -\widehat{A}^\top\\
				\widehat{\mat{A}} & \widehat{\mat{C}}
			\end{bmatrix}
			\begin{bmatrix}
				\theta \\
				w
			\end{bmatrix} 
			- 
			\begin{bmatrix}
				0\\
				\widehat{b}
			\end{bmatrix} .
\end{align*} 
Letting $\beta=\sigma_w/\sigma_{\theta}$, we have
	\begin{align*}
		\begin{bmatrix}
			\theta_{m+1} \\
			w_{m+1}
		\end{bmatrix} 
			&= 
				\begin{bmatrix}
					\theta_{m}\\
					w_m
				\end{bmatrix} 
				- \sigma_{\theta} 
				\left(
					\begin{bmatrix}
						\rho I & -\widehat{A}^\top\\
						\beta\widehat{\mat{A}} & \beta\widehat{\mat{C}}
					\end{bmatrix}
					\begin{bmatrix}
						\theta_{m} \\
						w_{m}
					\end{bmatrix} 
					\!\!-\!\!
					\begin{bmatrix}
						0\\
						\beta\widehat{b}
					\end{bmatrix} 
				\right) .
	\end{align*}
Subtracting both sides of the above recursion by $(\theta_{\star}, w_{\star})$ and using \eqref{eqn:qp_first_order_necessary}, we obtain
	\begin{align*}
		\begin{bmatrix}
			\theta_{m+1} - \theta_\star \\
			w_{m+1} \!-\! w_\star
		\end{bmatrix} 
			&= 
				\begin{bmatrix}
					\theta_{m} - \theta_\star\\
					w_m \!-\! w_\star
				\end{bmatrix} 
				\!\!-\!
				\sigma_\theta
				\begin{bmatrix}
					\rho I & - \widehat{A}^T \\
					\beta \widehat{A} & \beta \widehat{C}
				\end{bmatrix} \!\!
				\begin{bmatrix}
					\theta_{m} \!-\! \theta_\star\\
					w_m \!-\! w_\star
				\end{bmatrix} .
	\end{align*}

We analyze the convergence of the algorithms by examining the differences between the current parameters to the optimal solution. 
More specifically, we define a scaled residue vector
\begin{align}
	\Delta_m 
		&\triangleq 
			\begin{bmatrix}
				\theta_m - \theta_\star \\
				\frac{1}{\sqrt{\beta}}(w_m - w_\star) 
			\end{bmatrix},
			\label{eqn:current_optimal_difference}
\end{align}
which obeys the following iteration:
\begin{align}
\Delta_{m+1} &= \left(\mat{I}-\sigma_\theta\mat{G}\right)\Delta_{m}, 
\label{eqn:qp_gd_dynamics}
\end{align}
where $G$ is exactly the matrix defined in \eqref{Equ:Appendix:G_def}. 
As analyzed in Section~\ref{Appendix:DiagonalizabilityOfG}, if we choose
~$\beta$ sufficiently large, such as in~\eqref{Equ:Appendix:beta_Value},
then $G$ is diagonalizable with all its eigenvalues real and positive.
In this case, we let $Q$ be the matrix of eigenvectors in the eigenvalue
decomposition $\mat{G} = \mat{Q}\mat{\Lambda}\mat{Q}^{-1}$,
and use the potential function 
\begin{align*}
	P_m \triangleq \norm{\mat{Q}^{-1}\Delta_m}^2_2
\end{align*}
in our convergence analysis. 
We can bound the usual Euclidean distance by $P_m$ as
\begin{align*}
\| \theta_m - \theta_{\star} \|^2 + \| w_m - w_{\star} \|^2 
&\le	(1+\beta) \sigma_{\max}^2(Q) P_m.
\end{align*}
If we have linear convergence in $P_m$, 
then the extra factor $(1+\beta) \sigma_{\max}^2(Q)$ will appear inside 
a logarithmic term.

\textbf{Remark:} 
This potential function has an intrinsic geometric interpretation.
We can view column vectors of $\mat{Q}^{-1}$ a basis for the vector space,
which is \emph{not} orthogonal.
Our goal is to show that in this coordinate system, the distance to optimal solution shrinks at every iteration.

We proceed to bound the growth of $P_m$:
\begin{align}
	P_{m+1} &= \norm{\mat{Q}^{-1}\Delta_{m+1}}_2^2 \nonumber\\
	& = \norm{\mat{Q^{-1}\left(\mat{I} - \sigma_\theta\mat{G}\right)\Delta_{m}}}_2^2 \nonumber\\
	& = \norm{\mat{Q}^{-1}\left(\mat{Q}\mat{Q}^{-1} - \sigma_\theta\mat{Q}\mat{\Lambda}\mat{Q}^{-1}\right)\Delta_m}_2^2 \nonumber\\
	& = \norm{\left(\mat{I} - \sigma_\theta\mat{\Lambda}\right)\mat{Q}^{-1}\Delta_m}_2^2 \nonumber\\
	& \le \norm{\mat{I} - \sigma_\theta\mat{\Lambda}}_2^2 \norm{\mat{Q}^{-1}\Delta_m}_2^2 \nonumber\\
	& = \norm{\mat{I} - \sigma_\theta\mat{\Lambda}}_2^2 P_m 
	\label{eqn:qp_gd_potential_iterative}
\end{align}
The inequality above uses sub-multiplicity of spectral norm. 
We choose $\sigma_{\theta}$ to be 
\begin{align} \label{eqn:stepsize-theta}
	\sigma_\theta 
		= 
				\frac{1}{\lambda_{\max}\left(\mat{\Lambda}\right)}
		=		
				\frac{1}{\lambda_{\max}(G)},
\end{align}
Since all eigenvalues of $G$ are real and positive, we have
	\begin{align}
		\| I - \sigma_{\theta} \Lambda \|^2
			&=
				\left(1 - \frac{\lambda_{\min}(G)}{\lambda_{\max}(G)}\right)^2
				\nn\\
			&\le
				\left(
					1 
					-
					\frac{8}{81}					
					\cdot
					\frac{1}{\kappa(\widehat{C}) \kappa(\rho I + \widehat{A}^T \widehat{C}^{-1} \widehat{A})}
				\right)^2 ,
				\nn
	\end{align}
where we used the bounds on the eigenvalues $\lambda_{\max}(G)$ and
$\lambda_{\min}(G)$ in~\eqref{eqn:qp_eigenvalue_upper}
and~\eqref{Equ:Appendix:lambda_min_G_LB} respectively.
Therefore, we can achieve an $\epsilon$-close solution with 
\begin{align*}
	m 
		= 
			O\left( 
				\kappa(\widehat{C}) \kappa(\rho I + \widehat{A}^T \widehat{C}^{-1} \widehat{A}) 
				\log\left(\frac{P_0}{\epsilon}\right)
			\right)
\end{align*} 
iterations of the PDBG algorithm.

In order to minimize $\|I-\sigma_\theta\Lambda\|$, we can choose
\[
  \sigma_\theta = \frac{2}{\lambda_{\max}(G)+\lambda_{\min}(G)},
\]
which results in $\|I-\sigma_\theta\Lambda\|=1-2/(1+\kappa(\Lambda))$
instead of $1-1/\kappa(\Lambda)$.
The resulting complexity stays the same order.

The step sizes stated in Theorem~\ref{thm:pdbg} is obtained by replacing 
$\lambda_{\max}$ in~\eqref{eqn:stepsize-theta} with its upper bound
in~\eqref{eqn:qp_eigenvalue_upper} and setting $\sigma_w$ through the ratio 
$\beta=\sigma_w/\sigma_\theta$ as in~\eqref{Equ:Appendix:beta_Value}.

\section{Analysis of SVRG}
\label{Appendix:Proof_SVRG}

Here we establish the linear convergence of the SVRG algorithm
for policy evaluation described in Algorithm~\ref{algo:SVRG_saddle}.

Recall the finite sum structure in $\widehat{\mat{A}}$, $\widehat{b}$ 
and $\widehat{\mat{C}}$:
\begin{align*}
	\widehat{\mat{A}} = \frac{1}{n}\sum_{t=1}^{n} \mat{A}_t, \quad 
	\widehat{b} = \frac{1}{n}\sum_{t=1}^{n} b_t,
	\quad
	\widehat{\mat{C}} = \frac{1}{n}\sum_{t=1}^{n} \mat{C}_t .
\end{align*}
This structure carries over to the Lagrangian $\mathcal{L}(\theta,w)$
as well as the gradient operator $B(\theta, w)$, so we have
\[
  B(\theta,w) = \frac{1}{n}\sum_{t=1}^n B_t(\theta,w),
\]
where
\begin{align}
	B_t(\theta,w)
		 &= 
			\begin{bmatrix}
				  \rho I & -\mat{A}_t^\top\\
				  \mat{A}_t & \mat{C}_t
			\end{bmatrix}
			\begin{bmatrix} 
				\theta \\
				w
			\end{bmatrix} 
			- 
			\begin{bmatrix}
				0\\
				b_t
			\end{bmatrix} .
	\label{eqn:qp_Gt_def}
\end{align}

Algorithm~\ref{algo:SVRG_saddle} has both an outer loop and an inner loop.
We use the index~$m$ for the outer iteration and~$j$ for the inner iteration.
Fixing the outer loop index~$m$, we look at the inner loop of 
Algorithm~\ref{algo:SVRG_saddle}.
Similar to full gradient method, we first simplify the dynamics of SVRG.
\begin{align*}
	\begin{bmatrix}
	\theta_{m,j+1} \\
	w_{m,j+1}
	\end{bmatrix} 
	&=
		\begin{bmatrix}
			\theta_{m,j}\\
			w_{m,j}
		\end{bmatrix} 
		-
		\begin{bmatrix}
			\sigma_{\theta} & \\
			& \sigma_{w}
		\end{bmatrix}
		\times
		\bigg(B(\theta_{m-1},w_{m-1}) \\
        & \qquad
        + B_{t_j}(\theta_{m,j},w_{m,j}) - B_t(\theta_{m-1},w_{m-1})\bigg)
        \\
	&=
		\begin{bmatrix}
			\theta_{m,j}\\
			w_{m,j}
		\end{bmatrix} 
		-
		\begin{bmatrix}
			\sigma_{\theta} & \\
			& \sigma_{w}
		\end{bmatrix}
		\nn\\
		&\quad \times
		\Bigg(
			\begin{bmatrix}
				\rho I & -\widehat{A}^\top\\
				\widehat{\mat{A}} & \widehat{\mat{C}}
			\end{bmatrix}
			\begin{bmatrix}
				\theta_{m-1} \\
				w_{m-1}
			\end{bmatrix} 
			- 
			\begin{bmatrix}
				0\\
				\widehat{b}
			\end{bmatrix} 
			\nn\\
			&\qquad
			+
			\begin{bmatrix}
				\rho I & -A_t^\top\\
				A_t & C_t
			\end{bmatrix}
			\begin{bmatrix}
				\theta_{m,j} \\
				w_{m,j}
			\end{bmatrix} 
			- 
			\begin{bmatrix}
				0\\
				b_t
			\end{bmatrix} 
			\nn\\
			&\qquad
			-
			\begin{bmatrix}
				\rho I & -A_t^\top\\
				A_t & C_t
			\end{bmatrix}
			\begin{bmatrix}
				\theta_{m-1} \\
				w_{m-1}
			\end{bmatrix} 
			+
			\begin{bmatrix}
				0\\
				b_t
			\end{bmatrix} 
		\Bigg).
\end{align*} 
Subtracting $(\theta_\star, w_\star)$ from both sides and using the optimality condition \eqref{eqn:qp_first_order_necessary}, we have
\begin{align*}
	\begin{bmatrix}
	\theta_{m,j+1} - \theta_{\star} \\
	w_{m,j+1} - w_{\star}
	\end{bmatrix} 
	&=
	\begin{bmatrix}
			\theta_{m,j} - \theta_{\star}\\
			w_{m,j} - w_{\star}
		\end{bmatrix} 
		-
		\begin{bmatrix}
			\sigma_{\theta} & \\
			& \sigma_{w}
		\end{bmatrix}
		\nn\\
		&\quad \times
		\Bigg(
			\begin{bmatrix}
				\rho I & -\widehat{A}^\top\\
				\widehat{\mat{A}} & \widehat{\mat{C}}
			\end{bmatrix}
			\begin{bmatrix}
				\theta_{m-1} - \theta_{\star} \\
				w_{m-1} - w_{\star}
			\end{bmatrix} 
			\nn\\
			&\qquad
			+
			\begin{bmatrix}
				\rho I & -A_t^\top\\
				A_t & C_t
			\end{bmatrix}
			\begin{bmatrix}
				\theta_{m,j} - \theta_{\star} \\
				w_{m,j} - w_{\star}
			\end{bmatrix} 
			\nn\\
			&\qquad
			-
			\begin{bmatrix}
				\rho I & -A_t^\top\\
				A_t & C_t
			\end{bmatrix}
			\begin{bmatrix}
				\theta_{m-1} - \theta_{\star} \\
				w_{m-1} - w_{\star}
			\end{bmatrix} 
		\Bigg).
\end{align*} 
Multiplying both sides of the above recursion by $\diag(I,1/\sqrt{\beta}I)$, 
and using a residue vector $\Delta_{m,j}$ defined similarly as
in~\eqref{eqn:current_optimal_difference}, we obtain
\begin{align}
\Delta_{m,j+1} 
	&= 
		\Delta_{m,j} - \sigma_{\theta} ( G \Delta_{m-1} + G_{t_j} \Delta_{m,j} - G_{t_j} \Delta_{m-1})
		\nn\\
	&= 
		\left(\mat{I}-\sigma_\theta\mat{G}\right)\Delta_{m,j} 
        \nn\\
		&\qquad+ 
		\sigma_\theta\left(\mat{G}-\mat{G}_{t_j}\right)\left(\Delta_{m,j}-\Delta_{m-1}\right), 
\label{eqn:qp_svrg_dynamics}
\end{align}
where $G_{t_j}$ is defined in \eqref{eqn:LG-def}. 

For SVRG, we use the following potential functions to facilitate our analysis:
\begin{align}
  P_m &\triangleq \expect{\norm{\mat{Q}^{-1}\Delta_m}^2}, \label{eqn:Pm}\\
  P_{m,j} &\triangleq \expect{\norm{\mat{Q}^{-1}\Delta_{m,j}}^2}. \label{eqn:Pmj}
\end{align}
Unlike the analysis for the batch gradient methods, the non-orthogonality
of the eigenvectors will lead to additional dependency of the iteration 
complexity on the condition number of~$Q$, for which we give a bound 
in~\eqref{Equ:Appendix:CondNum_H}.

Multiplying both sides of Eqn.~\eqref{eqn:qp_svrg_dynamics} by $\mat{Q}^{-1}$, taking squared 2-norm and taking expectation, we obtain
\begin{align}
	P_{m,j+1} 
		&= 
			\E\Big[
				\big\|
					\mat{Q}^{-1}
					\big[
						\left(
							\mat{I}
							-
							\sigma_\theta\mat{G}
						\right)\Delta_{m,j}
						\nn\\
						&\qquad+
						\sigma_\theta
						\left(
							\mat{G}-\mat{G}_{t_j}
						\right)
						\left(
							\Delta_{m,j}
							-
							\Delta_{m-1}
						\right)
					\big]
				\big\|^2
			\Big] 
			\nonumber\\
		&\overset{(a)}{=} 
			\E\Big[
				\norm{\left(\mat{I}-\sigma_\theta\mat{\Lambda}\right)\mat{Q}^{-1}\Delta_{m,j}}^2
			\Big] 
			\nn\\
			&\quad
			+ 
			\sigma_\theta^2\;
			\E\Big[
              \norm{Q^{-1}\left(G\!-\!\mat{G}_{t_j}\right)
				\left(\Delta_{m,j}\!-\!\Delta_{m-1}\right)}^2
			\Big]
			\nonumber\\
		&\overset{(b)}{\le  }
			\norm{\mat{I}-\sigma_\theta\mat{\Lambda}}^2
            \E\Bigl[\norm{\mat{Q}^{-1}\Delta_{m,j}}^2\Bigr]
			\nn\\
			&\quad
			+ 
			\sigma_\theta^2\;
			\E\Big[
              \norm{\mat{Q}^{-1} G_{t_j} \left(\Delta_{m,j}-\Delta_{m-1}\right)}^2
			\Big] 
            \nn\\
		&\overset{(c)}{=}
        \norm{\mat{I}-\sigma_\theta\mat{\Lambda}}^2 P_{m,j}
			\nn\\
			&\quad
			+ 
			\sigma_\theta^2\;
			\E\Big[
              \norm{\mat{Q}^{-1} G_{t_j} \left(\Delta_{m,j}-\Delta_{m-1}\right)}^2
			\Big] .
            \label{eqn:Pmj-to-be-continued}
\end{align}
where step (a) used the facts that $G_{t_j}$ is independent of $\Delta_{m,j}$
and $\Delta_{m-1}$ and $\E[G_{t_j}]=G$ so the cross terms are zero,
step~(b) used again the same independence 
and that the variance of a random variable is less than its second moment,
and step~(c) used the definition of $P_{m,j}$ in~\eqref{eqn:Pmj}.
To bound the last term in the above inequality, we use the simple notation
$\delta = \Delta_{m,j}-\Delta_{m-1}$ and have
\begin{align*}
\norm{\mat{Q}^{-1} G_{t_j} \delta}^2
&= \delta^T G_{t_j}^T Q^{-T} Q^{-1} G_{t_j} \delta\\
&\leq \lambda_{\max}(Q^{-T}Q^{-1}) \delta^T G_{t_j}^T G_{t_j} \delta.
\end{align*}
Therefore, we can bound the expectation as
\begin{align}
&  \E\bigl[\norm{\mat{Q}^{-1} G_{t_j} \delta}^2\bigr]\nn\\
 \leq& \lambda_{\max}(Q^{-T}Q^{-1}) \E\bigl[\delta^T G_{t_j}^T G_{t_j} \delta\bigr]\nn\\
  =& \lambda_{\max}(Q^{-T}Q^{-1}) \E\bigl[\delta^T \E[G_{t_j}^T G_{t_j}] \delta\bigr]\nn\\
  \leq& \lambda_{\max}(Q^{-T}Q^{-1}) L_G^2 \E\bigl[\delta^T \delta\bigr]\nn\\
  =& \lambda_{\max}(Q^{-T}Q^{-1}) L_G^2 \E\bigl[\delta^T Q^{-T} Q^T Q Q^{-1} \delta\bigr]\nn\\
  =& \lambda_{\max}(Q^{-T}Q^{-1}) \lambda_{\max}(Q^T Q) L_G^2 \E\bigl[\delta^T Q^{-T} Q^{-1} \delta\bigr]\nn\\
\leq& \kappa(Q)^2 L_G^2 \E \bigl[\|Q^{-1}\delta\|^2\bigr],
\label{eqn:pull-G-out}
\end{align}
where in the second inequality we used the definition of $L_G^2$ 
in~\eqref{eqn:LG-def}, i.e., $L_G^2=\|\E[G_{t_j}^T G_{t_j}]\|$.
In addition, we have
\begin{align*}
\E\bigl[\|Q^{-1}\delta\|^2\bigr]
=&\E\bigl[\norm{\mat{Q}^{-1}(\Delta_{m,j}-\Delta_{m-1})}^2\bigr]\\
\leq & 2\;\E\bigl[\norm{\mat{Q}^{-1}\Delta_{m,j}}^2\bigr] 
      +2\;\E\bigl[\norm{\mat{Q}^{-1}\Delta_{m-1}}^2\bigr]\\
= &\; 2 P_{m,j} + 2 P_{m-1}.
\end{align*}
Then it follows from~\eqref{eqn:Pmj-to-be-continued} that
\begin{align*}
  P_{m,j+1} \leq& \|I-\sigma_\theta \Lambda\|^2 P_{m,j} \\
 & +2\sigma_\theta^2 \kappa^2(Q) L_G^2 (P_{m,j}+P_{m-1}).
\end{align*}

Next, let $\lambda_{\max}$ and $\lambda_{\min}$ 
denote the largest and smallest diagonal elements of $\mat{\Lambda}$ 
(eigenvalues of~$G$), respectively.
Then we have
\begin{align*}
\norm{\mat{I}-\sigma_\theta\mat{\Lambda}}^2
&=\max\left\{ (1-\sigma_{\theta}\lambda_{\min})^2,~(1-\sigma_{\theta}\lambda_{\min})^2\right\} \\
&\leq 1 - 2\sigma_{\theta}\lambda_{\min}+ \sigma^2_{\theta}\lambda^2_{\max}\\
&\leq 1 - 2\sigma_{\theta}\lambda_{\min}+ \sigma^2_{\theta}\kappa^2(Q) L_G^2,
\end{align*}
where the last inequality uses the relation 
\[
  \lambda_{\max}^2 \!\le\! \|G\|^2 =\! \| \E G_t
\|^2 \le \|\E G_t^T G_t\|\! =  L_G^2 \le \kappa^2(Q) L_G^2.
\]
It follows that
\begin{align}
  P_{m,j+1}
  &\le \bigl(1-2\sigma_{\theta}\lambda_{\min}+\sigma_\theta^2\kappa^2\left(\mat{Q}\right)L_G^2 \bigr) P_{m,j} 
			\nn\\
			&\quad+ 
			2\sigma_\theta^2 \;\kappa^2\left(\mat{Q}\right)L_G^2 (P_{m,j} + P_{m-1})
			\nn\\
		&=
			 \left[
			 	1-2\sigma_{\theta}\lambda_{\min} 
				+ 
				3\sigma_\theta^2
					\kappa^2\!\left(\mat{Q}\right)L_G^2
			\right]P_{m,j} 
			\nn\\
			&\quad+ 
			2\sigma_\theta^2\;\kappa^2\!\left(\mat{Q}\right)L_G^2 P_{m-1}
      \nn
\end{align}
If we choose $\sigma_\theta$ to satisfy 
\begin{align}\label{eqn:svrg-step-bound}
0<\sigma_\theta \leq \frac{\lambda_{\min}}{3\kappa^2\left(\mat{Q}\right)L_G^2},
\end{align}
then $3\sigma_\theta^2\kappa^2\!\left(\mat{Q}\right)L_G^2 
< \sigma_{\theta} \lambda_{\min}$, which implies
\begin{align}
  P_{m,j+1}
		&\leq
			 \left(
			 	1-\sigma_{\theta}\lambda_{\min} 
			\right)P_{m,j} 
			+ 2\sigma_\theta^2\;\kappa^2\!\left(\mat{Q}\right)L_G^2 P_{m-1} .
      \nn
\end{align}

Iterating the above inequality over $j=1,\cdots,N-1$ and 
using $P_{m,0} = P_{m-1}$ and $P_{m,N}=P_m$, we obtain
\begin{align}
&	P_m 
		= 
			P_{m,N} \nn \\
		&\le 
			\bigg[\!\bigl(1\!-\!\sigma_\theta \lambda_{\min}\bigr)^N 
			\!\!\!+\! 2\sigma_\theta^2\kappa^2\!\left(\mat{Q}\right)\!L_G^2\!
    \sum_{j=0}^{N-1}\!\bigl(1\!-\!\sigma_\theta \lambda_{\min}\bigr)^j \bigg]\! P_{m-1}
			\nn \\
		&= 
			\bigg[\!\bigl(1\!-\!\sigma_\theta \lambda_{\min}\bigr)^N 
			\!\!\!+\! 2\sigma_\theta^2\kappa^2\!\left(\mat{Q}\right)\!L_G^2\!
    \frac{1\!-\!(1\!-\!\sigma_\theta \lambda_{\min})^N}{1\!-\!(1\!-\!\sigma_\theta \lambda_{\min})} \bigg]\! P_{m-1}
			\nn \\
    &\leq 
			\bigg[\bigl(1-\sigma_\theta \lambda_{\min}\bigr)^N 
      + \frac{2\sigma_\theta^2\kappa^2\!\left(\mat{Q}\right)L_G^2}{\sigma_\theta\lambda_{\min}}\bigg]  P_{m-1} \nn \\
    &= 
			\bigg[\bigl(1-\sigma_\theta \lambda_{\min}\bigr)^N 
      + \frac{2\sigma_\theta\kappa^2\!\left(\mat{Q}\right)L_G^2}{\lambda_{\min}}\bigg]  P_{m-1} .
      \label{eqn:svrg-Pm-bound}
\end{align}
We can choose
\begin{equation}\label{eqn:svrg-stepsize-N}
  \sigma_\theta = \frac{\lambda_{\min}}{ 5\kappa^2(Q)L_G^2}, \quad
  N=\frac{1}{\sigma_\theta \lambda_{\min}}
  = \frac{5\kappa^2(Q)L_G^2}{\lambda_{\min}^2}, 
\end{equation}
which satisfies the condition in~\eqref{eqn:svrg-step-bound} and results in 
\[
  P_m \leq (e^{-1} + 2/5) P_{m-1} \leq (4/5) P_{m-1}.
\]
There are many other similar choices, for example,
\[
  \sigma_\theta = \frac{\lambda_{\min}}{ 3\kappa^2\!(Q)L_G^2}, \quad
  N=\frac{3}{\sigma_\theta \lambda_{\min}}
  = \frac{9\kappa^2(Q)L_G^2}{\lambda_{\min}^2}, 
\]
which results in 
\[
  P_m \leq (e^{-3} + 2/3) P_{m-1} \leq (3/4) P_{m-1}.
\]
These results imply that the number of outer iterations needed to have 
$\E[P_m]\leq \epsilon]$ is $\log(P_0/\epsilon)$.
For each outer iteration, the SVRG algorithm need $O(nd)$ operations to 
compute the full gradient operator $B(\theta,w)$, and then 
$N=O(\kappa^2(Q)L_G^2/\lambda^2_{\min})$ inner iterations with each
costing $O(d)$ operations.
Therefore the overall computational cost is
\begin{align*}
O\left(\left(n+\frac{\kappa^2\left(\mat{Q}\right)L_G^2}{\lambda_{\min}^2}\right)d\; \log\left(\frac{P_0}{\epsilon}\right)\right) .
\end{align*} 
Substituting~\eqref{Equ:Appendix:CondNum_H} 
and~\eqref{Equ:Appendix:lambda_min_G_LB} in the above bound, 
we get the overall cost estimate
\begin{align*}
	O\left(
		\left(
			n
			+
			\frac{\kappa(\widehat{C})L_G^2}
			{\lambda_{\min}^2(\rho I + \widehat{A}^T \widehat{C}^{-1} \widehat{A})}\right)d\;\log\left(\frac{P_0}{\epsilon}
		\right)
	\right).
\end{align*}

Finally, substituting the bounds in~\eqref{Equ:Appendix:CondNum_H} 
and~\eqref{Equ:Appendix:lambda_min_G_LB} into~\eqref{eqn:svrg-stepsize-N},
we obtain the $\sigma_\theta$ and $N$ stated in Theorem~\ref{thm:svrg}:
\begin{align*}
  \sigma_\theta &= \frac{\lambda_{\min}(\rho I +\widehat{A}^T \widehat{C}^{-1} \widehat{A})}{ 48\kappa(\widehat{C})L_G^2}, \\
  N&= \frac{51 \kappa^2(\widehat{C})L_G^2}{\lambda_{\min}^2(\rho I +\widehat{A}^T \widehat{C}^{-1} \widehat{A})}, 
\end{align*}
which achieves the same complexity.

\section{Analysis of SAGA}
\label{Appendix:Proof_SAGA}
SAGA in Algorithm~\ref{algo:saga_saddle} maintains a table of previously computed gradients.
Notation wise, we use $\phi_t^m$ to denote that at $m$-th iteration, $g_t$ is computed using $\theta_{\phi_t^m}$ and $w_{\phi_t^m}$.
With this definition, $\phi_t^m$ has the following dynamics:\begin{align}
\phi_t^{m+1} = \begin{cases}
\phi_t^m\quad \text{if }t_m \neq t,\\
m ~~\quad \text{if }t_m = t .
\end{cases}
\label{eqn:qp_saga_phi_t^m}
\end{align}
We can write the $m$-th iteration's full gradient as
\[
B = \frac{1}{n}\sum_{t=1}^{n}B_t\left(\theta_{\phi_t^m},w_{\phi_t^m}\right).
\]
For convergence analysis, we define the following quantity:
\begin{align}
	\Delta_{\phi_t^m} 
		\triangleq 
					\begin{bmatrix}
						\theta_{\phi_t^m} - \theta_\star \\
						\frac{1}{\sqrt{\beta}}(w_{\phi_t^m} - w_\star)
					\end{bmatrix}. 
					\label{eqn:qp_saga_Delta_phi}
\end{align} 
Similar to~\eqref{eqn:qp_saga_phi_t^m}, it satisfies the following iterative relation:
\begin{align*}
\Delta_{\phi_{t}^{m+1}} &= \begin{cases}
\Delta_{\phi_{t}^m} \quad \text{if} \quad t_m\neq t,\\
\Delta_{m} ~~\quad \text{if} \quad t_m = t.
\end{cases}
\end{align*}
With these notations, we can express the vectors used in SAGA as
\begin{align*}
 B_m & =  \frac{1}{n}	\sum_{t=1}^n \begin{bmatrix}
			\rho I & - A_t^T \\	A_t & C_t \end{bmatrix} 
		\begin{bmatrix}	\theta_{\phi_t^m} \\ w_{\phi_t^m} \end{bmatrix}
		- \frac{1}{n} \sum_{t=1}^n	\begin{bmatrix}	0 \\ b_t\end{bmatrix},\\
h_{t_m} &= \begin{bmatrix}
  \rho I & - A_{t_m}^T \\	A_{t_m} & C_{t_m} \end{bmatrix} 
		\begin{bmatrix}	\theta_m \\ w_m \end{bmatrix}
        - \begin{bmatrix}	0 \\ b_{t_m}\end{bmatrix},\\
g_{t_m} &= \begin{bmatrix}
  \rho I & - A_{t_m}^T \\	A_{t_m} & C_{t_m} \end{bmatrix} 
		\begin{bmatrix}	\theta_{\phi_t^m} \\ w_{\phi_t^m} \end{bmatrix}
        - \begin{bmatrix}	0 \\ b_{t_m}\end{bmatrix}.
\end{align*}
The dynamics of SAGA can be written as
\begin{align*}
\begin{bmatrix}
		\theta_{m+1} \\
		w_{m+1}
\end{bmatrix} 
&= \begin{bmatrix}
		\theta_m \\
		w_m
		\end{bmatrix}  
		-  
		\begin{bmatrix}
		\sigma_\theta & \\
		& \sigma_w
		\end{bmatrix}\left(B_m + h_{t_m} - g_{t_m}\right)\\
&=
	\begin{bmatrix}
		\theta_m \\
		w_m
	\end{bmatrix}  
	\!-  \!
	\begin{bmatrix}
		\sigma_\theta & \\
		& \sigma_w
	\end{bmatrix} \\
    \nn
    &\quad
	\Bigg\{
		\frac{1}{n}
		\sum_{t=1}^n
		\begin{bmatrix}
			\rho I & - A_t^T \\
			A_t & C_t
		\end{bmatrix} \!
		\begin{bmatrix}
			\theta_{\phi_t^m} \\
			w_{\phi_t^m}
		\end{bmatrix}
		+
		\frac{1}{n}
		\sum_{t=1}^n
		\begin{bmatrix}
			0 \\
			b_t
		\end{bmatrix} \\
    &\quad +\!
	\begin{bmatrix}
		\rho I\!\!\!\!\!\! & -A_{t_m}^T \\
		A_{t_m}\!\!\!\!\!\! & C_{t_m}
	\end{bmatrix}\!
	\begin{bmatrix}
		\theta_m \\
		w_m
	\end{bmatrix}
	\!-\!
	\begin{bmatrix}
		\rho I \!\!\!\!\!\! & -A_{t_m}^T \\
		A_{t_m} \!\!\!\!\!\! & C_{t_m}
	\end{bmatrix} \!
	\begin{bmatrix}
		\theta_{\phi_{t_m}^m} \\
		w_{\phi_{t_m}^m}
	\end{bmatrix}\!
	\Bigg\}
\end{align*}
Subtracting $(\theta_\star, w_\star)$ from both sides, and using the optimality
condition in~\eqref{eqn:qp_first_order_necessary}, we obtain
\begin{align*}
\begin{bmatrix}
		\theta_{m+1}-\theta_\star \\
		w_{m+1}-w_\star
\end{bmatrix} 
&= \begin{bmatrix}
		\theta_m - \theta_\star \\
		w_m - w_\star
		\end{bmatrix}  
	-  
	\begin{bmatrix}
		\sigma_\theta & \\
		& \sigma_w
	\end{bmatrix} \\
    \nn
    &\quad
	\Bigg\{
		\frac{1}{n}
		\sum_{t=1}^n
		\begin{bmatrix}
			\rho I & - A_t^T \\
			A_t & C_t
		\end{bmatrix} \!
		\begin{bmatrix}
			\theta_{\phi_t^m} - \theta_\star\\
			w_{\phi_t^m} - w_\star
		\end{bmatrix}\\
    &\quad +
	\begin{bmatrix}
		\rho I\!\!\!\!\!\! & -A_{t_m}^T \\
		A_{t_m}\!\!\!\!\!\! & C_{t_m}
	\end{bmatrix}\!
	\begin{bmatrix}
		\theta_m - \theta_\star \\
		w_m - w_\star
	\end{bmatrix} \\
    &\quad-
	\begin{bmatrix}
		\rho I \!\!\!\!\!\! & -A_{t_m}^T \\
		A_{t_m} \!\!\!\!\!\! & C_{t_m}
	\end{bmatrix} \!
	\begin{bmatrix}
		\theta_{\phi_{t_m}^m} - \theta_\star \\
		w_{\phi_{t_m}^m} - w_\star
	\end{bmatrix}\!
	\Bigg\}.
\end{align*}
Multiplying both sides by $\diag(I, 1/\sqrt{\beta}I)$, we get
\begin{align}
	\Delta_{m+1} 
		&= 
			\Delta_m - \left(\frac{\sigma_\theta}{n}\sum_{t=1}^{n}\mat{G}_t\Delta_{\phi_t^m}\right)
			\nn\\
			&\quad- \sigma_\theta\mat{G}_{t_m}\left(\Delta_m - \Delta_{\phi_{t_m}^m}\right). 
            \label{eqn:qp_saga_dynamics}
\end{align}
where $G_{t_m}$ is defined in \eqref{eqn:LG-def}. 

For SAGA, we use the following two potential functions:
\begin{align}
	P_m 
			&= 
				\E\norm{\mat{Q}^{-1}\Delta_m}_2^2, \nonumber\\
	Q_m 
			&= 
				\E\biggl[\!\frac{1}{n}\!\!\sum_{t=1}^n\norm{\mat{Q}^{-1} G_t \Delta_{\phi_t^m}}_2^2\biggr] 
			= 
            \E\biggl[\!\norm{\mat{Q}^{-1}G_{t_m}\Delta_{\phi_{t_m}^m}}_2^2\!\biggr].\nonumber
\end{align}
The last equality holds because we use uniform sampling.
We first look at how $P_m$ evolves.
To simplify notation, let 
\begin{align}
		v_m = \left(\frac{\sigma_\theta}{n}\sum_{t=1}^{n}\mat{G}_t\Delta_{\phi_t^m}\right) + \sigma_\theta\mat{G}_{t_m}\left(\Delta_m - \Delta_{\phi_{t_m}^m}\right),
		\nn
\end{align}
so that~\eqref{eqn:qp_saga_dynamics} becomes $\Delta_{m+1}=\Delta_m - v_m$.
We have
\begin{align*}
P_{m+1}
	& = 
		\expect{\norm{\mat{Q}^{-1}\Delta_{m+1}}_2^2} \nonumber\\
	& = 
		\E\Big[
				\big\|
					\mat{Q}^{-1}
					\left(\Delta_m -v_m \right)
				\big\|^2
		\Big]
		\nonumber\\
	& = 
		\E\Big[\!
			\norm{\mat{Q}^{-1}\!\Delta_m}_2^2 
			\!-\! 
			2\Delta_m^\top\mat{Q}^{-\top}\!\mat{Q}^{-1} v_m
			\!+\! 
			\norm{\mat{Q}^{-1}v_m}_2^2
		\Big] 
		\nonumber \\
	& = 
			P_m
      -\E\big[2\Delta_m^\top\mat{Q}^{-\top}\mat{Q}^{-1} v_m\big]
			+ 
			\E\Big[\norm{\mat{Q}^{-1}v_m}_2^2\Big] .
		\nonumber 
\end{align*}
Since $\Delta_m$ is independent of $t_m$, we have
\begin{align*}
  \E\Big[2\Delta_m^\top\mat{Q}^{-\top}\mat{Q}^{-1} v_m\Big]
  =\E\Big[2\Delta_m^\top\mat{Q}^{-\top}\mat{Q}^{-1} \E_{t_m}[v_m]\Big],
\end{align*}
where the inner expectation is with respect to $t_m$ conditioned on all
previous random variables. 
Notice that
\[
  \E_{t_m} \big[G_{t_m}\Delta_{\phi^m_{t_m}}\big] 
  = \frac{1}{n}\sum_{t=1}^n G_t \Delta_{\phi^m_t},
\]
which implies
$\E_{t_m} [v_m] = \sigma_\theta \E_{t_m}[G_{t_m}]\Delta_m 
=\sigma_\theta G \Delta_m$.
Therefore, we have
\begin{align}
  P_{m+1} 
  &= P_m - \E\Big[ 2\sigma_\theta \Delta_m^T Q^{-T}Q^{-1}G\Delta_m \Big]
  +\E\Big[\norm{\mat{Q}^{-1}v_m}_2^2\Big] \nn\\
  &= P_m - \E 2\sigma_\theta \Big[\Delta_m^T Q^{-T}\Lambda Q^{-1}\Delta_m \Big]
  +\E\Big[\norm{\mat{Q}^{-1}v_m}_2^2\Big] \nn\\
&\leq P_m - 2\sigma_\theta \lambda_{\min} \E\Big[\big\|Q^{-1}\Delta_m\big\|^2\Big]
  +\E\Big[\norm{\mat{Q}^{-1}v_m}_2^2\Big] \nn\\
	& = (1-2\sigma_\theta\lambda_{\min})P_m + \expect{ \norm{\mat{Q}^{-1}v_m}_2^2}, \label{eqn:qp_saga_pm}
\end{align} 
where the inequality used 
$\lambda_{\min}\!\triangleq\!\lambda_{\min}(\Lambda)\!=\!\lambda_{\min}(G)>0$, 
which is true under our choice of 
$\beta=\sigma_w/\sigma_\theta$ in Section~\ref{Appendix:DiagonalizabilityOfG}.
Next, we bound the last term of Eqn.~\eqref{eqn:qp_saga_pm}:
\begin{align}
&	\E\Big[\norm{\mat{Q}^{-1}v_m}_2^2\Big]
			\nn\\
		=\;& 
			\E
			\biggl[
				\Big\|
					\mat{Q}^{-1}
					\Big(						
							\frac{\sigma_\theta}{n}
							\sum_{t=1}^{n}
							\mat{G}_t\Delta_{\phi_t^m}
						\!+\! 
						\sigma_\theta\mat{G}_{t_m} \!\!
						\left(
							\Delta_m \!-\! \Delta_{\phi_{t_m}^m}
						\right)
					\Big)
				\Big\|^2
			\biggr] 
			\nonumber\\
		\le\; & 
			2\sigma_\theta^2\expect{\norm{\mat{Q}^{-1}\mat{G}_{t_m}\Delta_m}_2^2} 
			\nn\\
			&\quad + 
      2\sigma_\theta^2\E\biggl[\Bigl\|\mat{Q}^{-1}\Bigl(\frac{1}{n}\sum_{t=1}^{n}\mat{G}_t\Delta_{\phi_t^m}-\mat{G}_{t_m}\Delta_{\phi^m_{t_m}}\Bigr)\Bigr\|^2\biggr]
			\nonumber\\
		\le\; & 
			2\sigma_\theta^2 
				\E\Big[\!\norm{\mat{Q}^{-1}\!\mat{G}_{t_m}\Delta_m}_2^2\Big]
				\!+ 2\sigma_\theta^2\E\Big[\|Q^{-1}\mat{G}_{t_m}\!\Delta_{\phi^m_{t_m}}\|^2\Big] 
				\nn\\
		=\; &
			2\sigma_\theta^2 
				\E\Big[\!\norm{\mat{Q}^{-1}\!\mat{G}_{t_m}\Delta_m}_2^2\Big]
				\!+ 2\sigma_\theta^2 Q_m
				,
			\nonumber 
\end{align}
where the first inequality uses $\norm{a+b}_2^2 \le 2\norm{a}_2^2 + 2 \norm{b}_2^2$,
and the second inequality holds because for any random variable~$\xi$,
$\E\norm{\xi - \E\left[\xi\right]}_2^2 = \E\norm{\xi}_2^2 - \norm{\E\xi}_2^2 \le \E\norm{\xi}_2^2$.
Using similar arguments as in~\eqref{eqn:pull-G-out}, we have
\begin{align}
\E\Big[\norm{\mat{Q}^{-1}\!\mat{G}_{t_m}\Delta_m}_2^2\Big]
&\leq \kappa^2\!(Q) L_G^2 P_m,
\label{eqn:qp_saga_qgdelta}
\end{align}
Therefore, we have
\begin{align}
P_{m+1}
	&\le 
		\left(1-2\sigma_\theta\lambda_{\min} + 2\sigma_\theta^2\kappa^2\left(\mat{Q}\right)L_{\mat{G}}^2\right)P_m 
\nn\\
		&\quad+ 
		2\sigma_\theta^2\mat{Q}_m.  \label{eqn:qp_saga_pm_dynamics}
\end{align}

The inequality~\eqref{eqn:qp_saga_pm_dynamics} shows that the dynamics 
of $P_m$ depends on both $P_m$ itself and $Q_m$.
So we need to find another iterative relation for $P_m$ and $Q_m$.
To this end, we have
\begin{align}
	Q_{m+1} 
	& = 
		\expect{\frac{1}{n}\sum_{t=1}^{n}\norm{\mat{Q}^{-1} G_t \Delta_{\phi_t^{m+1}}}_2^2} \nonumber \\
	&=
		\E\bigg[
			\frac{1}{n}
			\| Q^{-1} G_{t_m} \Delta_{\phi_{t_m}^{m+1}}\|^2
			\nn\\
			&\qquad
			+
			\frac{1}{n}
			\sum_{t\neq t_m}
			\| Q^{-1} G_t \Delta_{\phi_t^{m+1}}\|^2
		\bigg]
		\nn\\
	&\overset{(a)}{=}
		\E\bigg[
			\frac{1}{n}
			\| Q^{-1} G_{t_m} \Delta_{m}\|^2
			\nn\\
		  &\qquad
			+
			\frac{1}{n}
			\sum_{t\neq t_m}
			\| Q^{-1} G_t \Delta_{\phi_t^{m}}\|^2
		\bigg]
		\nn\\
	&=		
		\E\bigg[
			\frac{1}{n}
			\| Q^{-1} G_{t_m} \Delta_{m}\|^2
			-
			\frac{1}{n}
			\| Q^{-1} G_{t_m} \Delta_{\phi_{t_m}^m}\|^2
			\nn\\
			&\qquad
			+
			\frac{1}{n}
			\sum_{t=1}^n
			\| Q^{-1} G_t \Delta_{\phi_t^{m}}\|^2
		\bigg]
		\nn\\
	&=
		\frac{1}{n}
		\E[\| Q^{-1} G_{t_m} \Delta_{m}\|^2]
		-
		\frac{1}{n}
		\E[\| Q^{-1} G_{t_m} \Delta_{\phi_{t_m}^m}\|^2]
		\nn\\
		&\qquad
		+
		\E\bigg[
			\frac{1}{n}
			\sum_{t=1}^n
			\| Q^{-1} G_t \Delta_{\phi_t^{m}}\|^2
		\bigg]
		\nn\\
	&=
		\frac{1}{n}
		\E[\| Q^{-1} G_{t_m} \Delta_{m}\|^2]
		-
		\frac{1}{n}
		\E[\| Q^{-1} G_{t_m} \Delta_{\phi_{t_m}^m}\|^2]
		\nn\\
		&\qquad
		+
		\E\bigg[
			\| Q^{-1} G_{t_m} \Delta_{\phi_{t_m}^{m}}\|^2
		\bigg]
		\nn\\
	&=
		\frac{1}{n}
		\E[\| Q^{-1} G_{t_m} \Delta_{m}\|^2]
		+
		\frac{n-1}{n}
		Q_m
		\nn\\
	&\overset{(b)}{\le}
		\frac{\kappa^2(Q) L_G^2}{n}P_m + \frac{n-1}{n}Q_m. \label{eqn:qp_saga_qm_dynamics}
\end{align}
where step (a) uses \eqref{eqn:qp_saga_phi_t^m} and step (b) uses \eqref{eqn:qp_saga_qgdelta}.

To facilitate our convergence analysis on $P_m$, we construct a new Lyapunov function which is a linear combination of Eqn.~\eqref{eqn:qp_saga_pm_dynamics} and Eqn.~\eqref{eqn:qp_saga_qm_dynamics}.
Specifically, consider \begin{align*}
	T_m = P_m + \frac{n\sigma_\theta\lambda_{\min}\left(1-\sigma_\theta\lambda_{\min}\right)}{\kappa^2(Q)L_G^2}Q_m. 
\end{align*}
Now consider the dynamics of $T_{m}$. We have
\begin{align*}
T_{m+1}
	& = 
		P_{m+1} 
		+ 
		\frac{n\sigma_\theta\lambda_{\min}
		\left(1-\sigma_\theta\lambda_{\min}\right)}{\kappa^2(Q)L_G^2}
		Q_{m+1}
		\\
	& \le 
		\left(1-2\sigma_\theta\lambda_{\min} + 2\sigma_\theta^2\kappa^2\left(\mat{Q}\right)L_G^2\right)P_m 
		+ 2\sigma_\theta^2\mat{Q}_m  
		\\
		& \quad\!+\! 
		\frac{n\sigma_\theta\lambda_{\min}\!\left(1\!-\!\sigma_\theta\lambda_{\min}\right)}{\kappa^2(Q)L_G^2}\!
		\left(\!\frac{\kappa^2(Q)L_G^2}{n}P_m \!+\! \frac{n\!-\!1}{n}Q_m\!\right) \\
	& = 
		\left(1-\sigma_\theta\lambda_{\min}+2\sigma_\theta^2\kappa^2\!\left(\mat{Q}\right)L_G^2 - \sigma_\theta^2\lambda_{\min}^2\right)P_m \\
		& \quad +\! 
		\frac{
			2\sigma_\theta^2\kappa^2\!\left(\mat{Q}\right)L_G^2
			\!+\!
			\left(n\!-\!1\right)\sigma_\theta
			\lambda_{\min}\!
			\left(1\!\!-\!\sigma_\theta\lambda_{\min}\right)
		}{\kappa^2(Q)L_G^2} \! Q_m.
\end{align*}
Let's define 
\[
  \rho = \sigma_\theta\lambda_{\min} - 2\sigma_\theta^2\kappa^2\!\left(\mat{Q}\right)L_G^2.
\]
The coefficient for $P_m$ in the previous inequality can be upper bounded 
by $1-\rho$ because $1-\rho-\sigma_\theta^2\lambda_{\min}^2 \leq 1-\rho$.
Then we have
\begin{align}
& T_{m+1}
		\nn\\
	&\le 
		\left(1-\rho\right) P_m 
		+ 
		\nn\\
		&\quad
		\frac{	
			2\sigma_\theta^2\kappa^2\left(\mat{Q}\right)L_G^2
			\!+\!
			\left(n\!-\!1\right)\sigma_\theta\lambda_{\min}
			\left(1\!-\!\sigma_\theta\lambda_{\min}\right)
		}{\kappa^2(Q)L_G^2}
		Q_m
		\nonumber \\
	&=
		\left(1-\rho\right)
		\left(
			P_m 
			+ 
			\frac{n\sigma_\theta\lambda_{\min}\left(1-\sigma_\theta\lambda_{\min}\right)}{\kappa^2(Q)L_G^2}
			Q_m
		\right) 
		\nonumber\\
		& \quad 
		+ 
		\sigma_{\theta}
		\frac{
			2 \sigma_{\theta} \kappa^2(Q)L_G^2
			+
			(n\rho-1)
			\lambda_{\min}
			(1-\sigma_{\theta}\lambda_{\min})
		}{\kappa^2(Q)L_G^2}
		Q_m 
    \nn\\
	&=
		\left(1-\rho\right) T_m
		\nonumber\\
		& \quad 
		+ 
		\sigma_{\theta}
		\frac{
			2 \sigma_{\theta} \kappa^2(Q)L_G^2
			+
			(n\rho-1)
			\lambda_{\min}
			(1-\sigma_{\theta}\lambda_{\min})
		}{\kappa^2(Q)L_G^2}
		Q_m. 
		\label{eqn:qp_saga_T_conv}
\end{align}
Next we show that with the step size 
\begin{align}
	\sigma_\theta = \frac{\lambda_{\min}}{3\left(\kappa^2\left(\mat{Q}\right)L_\mat{G}^2 + n\lambda_{\min}^2\right)}
	\label{Appendix:SAGAproof:sigma_theta_choice}
\end{align}
(or smaller),  the second term on the right-hand side 
of~\eqref{eqn:qp_saga_T_conv} is non-positive.
To see this, we first notice that with this choice of $\sigma_\theta$, we have
\begin{align*}
\frac{\lambda_{\min}^2}{9\left(\kappa^2\left(\mat{Q}\right)L_G^2 \!+\! n\lambda_{\min}^2\right)}
\le 
\rho 
\le \frac{\lambda_{\min}^2}{3\left(\kappa^2\left(\mat{Q}\right)L_G^2 \!+\! n\lambda_{\min}^2\right)},
\end{align*}
which implies
\begin{align*}
n \rho - 1 
\le \frac{n \lambda_{\min}^2}{3\left(\kappa^2\left(\mat{Q}\right)L_G^2 \!+\! n\lambda_{\min}^2\right)} -1
\leq \frac{1}{3}-1 =-\frac{2}{3}.
\end{align*}
Then, it holds that
	\begin{align}
	&	2 \sigma_{\theta} \kappa^2(Q)L_G^2
		+
		(n\rho-1)
		\lambda_{\min}
		(1-\sigma_{\theta}\lambda_{\min})
		\nn\\
	\le & 
					2 \sigma_{\theta} \kappa^2(Q)L_G^2
					-
					\frac{2}{3}
					\lambda_{\min}
					(1-\sigma_{\theta}\lambda_{\min})
					\nn\\
  = &
					-
						\frac{(6n-2)\lambda_{\min}^3}{9\left(\kappa^2(Q)L_G^2 + n\lambda_{\min}^2\right)}
			<0.	\nn
	\end{align}
Therefore \eqref{eqn:qp_saga_T_conv} implies
	\begin{align}
		T_{m+1}
			&\le 
					(1-\rho) T_{m} .
					\nn
	\end{align}
Notice that $P_m\leq T_m$ and $Q_0= P_0$. Therefore we have $T_0\leq 2P_0$ and
\[
  P_m \leq 2 (1-\rho)^m P_0 .
\]
Using~\eqref{Appendix:SAGAproof:sigma_theta_choice}, we have 
\begin{align*}
  \rho = \sigma_\theta\lambda_{\min}(G) - 2\sigma_\theta^2\kappa^2(Q)L_G^2 
  \geq \frac{\lambda_{\min}^2}{9\bigl(\kappa^2(Q)L_G^2+n\lambda_{\min}^2\bigr)} .
\end{align*}
To achieve $P_{m} \le \epsilon$, we need at most
\begin{align*}
m = O\left(\left(n+\frac{\kappa^2\left(\mat{Q}\right)L_G^2}{\lambda_{\min}^2}\right)\log\left(\frac{P_0}{\epsilon}\right)\right) 
\end{align*} 
iterations.
Substituting \eqref{Equ:Appendix:lambda_min_G_LB} and \eqref{Equ:Appendix:CondNum_H} in the above bound, we get the desired iteration complexity
\begin{align*}
	O\left(
		\left(
			n
			+
			\frac{\kappa(\widehat{C})L_G^2}
			{\lambda_{\min}^2(\rho I + \widehat{A}^T \widehat{C}^{-1} \widehat{A})}\right)\log\left(\frac{P_0}{\epsilon}
		\right)
	\right).
\end{align*}
Finally, using the bounds in~\eqref{Equ:Appendix:CondNum_H} 
and~\eqref{Equ:Appendix:lambda_min_G_LB}, we can replace the
step size in~\eqref{Appendix:SAGAproof:sigma_theta_choice}
by
\[
  \sigma_\theta = \frac{\mu_\rho}{3\left(8\kappa^2(\widehat{C})L_G^2 + n \mu_\rho^2\right)},
\]
where 
$\mu_\rho=\lambda_{\min}^2(\rho I + \widehat{A}^T \widehat{C}^{-1} \widehat{A})$
as defined in~\eqref{eqn:mu-rho}.

\end{document}